\documentclass[11pt]{article}

\usepackage[margin=1in]{geometry}
\usepackage{times}
\usepackage[numbers,sort&compress]{natbib}
\usepackage{hyphenat}
\usepackage[utf8]{inputenc}
\usepackage[T1]{fontenc}
\usepackage{hyperref}
\usepackage{url}
\usepackage{booktabs}
\usepackage{amsfonts}
\usepackage{nicefrac}
\usepackage{microtype}
\usepackage{xcolor}
\usepackage{multirow}
\usepackage{caption}
\usepackage{subcaption}
\usepackage{graphicx}
\usepackage{amsmath}
\usepackage{amssymb}
\usepackage{amsthm}
\usepackage{bm}
\usepackage{longtable}
\usepackage{pdflscape}
\usepackage{array}
\usepackage{tabularx}

\newtheorem{theorem}{Theorem}[section]
\newtheorem{proposition}[theorem]{Proposition}
\newtheorem{lemma}[theorem]{Lemma}

\theoremstyle{definition}
\newtheorem{definition}[theorem]{Definition}
\newtheorem{assumption}[theorem]{Assumption}

\theoremstyle{remark}

\usepackage{tikz}
\usetikzlibrary{positioning}
\usetikzlibrary{fit}

\title{Coupled Usage--Sense Processes: Temporal and Attributable Lexical Semantic Change}

\author{%
  Haruka Ezoe\\
  Graduate School of Information Science and Technology\\
  The University of Tokyo\\
  \and
  Ryohei Hisano \\
  Graduate School of Information Science and Technology\\
  The University of Tokyo\\
  The Canon Institute for Global Studies
}
\date{}

\begin{document}

\maketitle

%% FIN
\begin{abstract}
Lexical semantic change is usually summarized by a scalar distance between independently sampled period distributions. This measures how much a word changed, but does not reveal when it changed, which mechanisms and component movements carried the change, or which usages support the attribution. We introduce Coupled Usage--Sense Processes (CUSP), which derives these answers from a single marginal preserving temporal process. A hierarchical coupling relates contextual distributions through latent usage components, while Markov composition makes adjacent and longer span correspondences compatible. Displacement operators quantify change magnitude and timing, split variation exactly between movement of component centers and reorganization within components, and attribute it to transported component pairs. Word-local modes resolve distinct directions of change and their activity over time, while representative passages from attributed components ground the analysis in text. Under a Gaussian mixture specialization, we prove parametric recovery of the operators and squared distances. Synthetic experiments support the predicted rate. CUSP remains competitive on English and German DWUG and recovers controlled Janus profiles while maintaining compositionally coherent transport. A large corpus of US court opinions demonstrates transition, mode, and passage attribution in unlabeled natural text. CUSP thus makes magnitude, timing, mechanism, movement, modes, and textual evidence compatible views of one lexical history.
\end{abstract}

%% FIN
\section{Introduction}
\label{sec:introduction}

Computational lexical semantic change (LSC) studies how a word’s contextual usages evolve. Standard evaluations condense this history to a scalar distance between periods. A score cannot reveal when change occurred, whether it reflects movement between usage components or reorganization within them, which transitions carried it, or which usages support the attribution. A multi-period analysis should answer these questions as compatible views of one lexical history.

Prior work addresses individual parts of this problem. Aligned and dynamic embeddings track changes in a word's distributional representation across periods \citep{hamilton2016diachronic,bamler2017dynamic}. Contextual approaches compare individual usages directly or cluster them into usage types, then measure change with distances, divergences, or optimal transport \citep{giulianelli2020analysing,periti2024systematic,montariol2021transport}. Sense based methods track prevalence, cluster continuity, and contributing senses \citep{frermann2016bayesian,hu2019diachronic,periti2025studying,kolli2026wordcentered,kokosinskii2026granularity}, while diachronic similarity matrices characterize temporal patterns \citep{kiyama2025analyzing}. Unbalanced transport identifies gains and losses in usage mass \citep{kishino2025quantifying}, and embedding axes aid interpretation \citep{aida2025investigating,aida2025scdtour}. \citet{chen2026lexical} identify interpretability as a major unresolved challenge, noting that current models often cannot explain how or why meanings change and explicitly interpretable approaches remain at an early stage. Prior methods already track senses and attribute change (Table~\ref{tab:method_comparison_compact}). To our knowledge, no existing LSC method builds one process that preserves period marginals and makes correspondence consistent across time, exactly decomposes its change score by mechanism and transported component pair, and links its word-local modes and attributed movements to representative passages.

\begin{figure*}[t]
\centering
\definecolor{cuspnavy}{HTML}{17365D}
\definecolor{cuspblue}{HTML}{0072B2}
\definecolor{cuspsky}{HTML}{56B4E9}
\definecolor{cuspgreen}{HTML}{009E73}
\definecolor{cusporange}{HTML}{D98C00}
\definecolor{cusppurple}{HTML}{A95A91}
\begin{tikzpicture}[
>=latex,
node distance=0.28cm,
every node/.style={align=center},
stage/.style={rounded corners=3pt,text width=3.0cm,minimum height=1.05cm,font=\sffamily\small,line width=0.75pt},
readout/.style={rounded corners=3pt,text width=3.0cm,minimum height=1.42cm,font=\sffamily\footnotesize,line width=0.65pt},
flow/.style={->,draw=cuspnavy,line width=0.85pt},
bus/.style={draw=cuspnavy,line width=0.85pt}
]
\node[stage,draw=cuspblue,fill=cuspsky!14] (a)
{period specific\\usage distributions\\$\{\gamma_t^{(w)}\}_{t=1}^{T}$};

\node[stage,draw=cuspblue,fill=cuspblue!10,right=of a] (b)
{hierarchical coupling\\of components and usages};

\node[stage,draw=cuspgreen,fill=cuspgreen!10,right=of b] (c)
{coherent temporal process\\preserving every period marginal};

\node[stage,draw=cuspnavy,fill=cuspnavy,text=white,right=of c] (d)
{displacement operators\\$\mathbf M_{\bullet}^{(w)}(t,s)$};

\path (b.north) -- (c.north)
node[midway,above=0.18cm,font=\sffamily\bfseries\footnotesize,text=cuspnavy]
{CUSP: ONE COUPLED TEMPORAL PROCESS, FOUR COMPATIBLE LSC READOUTS};

\draw[flow] (a) -- (b);
\draw[flow] (b) -- (c);
\draw[flow] (c) -- (d);

\node[readout,draw=cuspblue,fill=cuspblue!6,below=0.78cm of a] (o1)
{{\color{cuspblue}\textbf{How much and when?}}\\
change magnitude and\\temporal profile};

\node[readout,draw=cusporange,fill=cusporange!7,below=0.78cm of b] (o2)
{{\color{cusporange}\textbf{Which mechanism?}}\\
movement between centers\\and variation within components};

\node[readout,draw=cuspgreen,fill=cuspgreen!6,below=0.78cm of c] (o3)
{{\color{cuspgreen}\textbf{Which component movement?}}\\
transported component pairs\\and supporting passages};

\node[readout,draw=cusppurple,fill=cusppurple!7,below=0.78cm of d] (o4)
{{\color{cusppurple}\textbf{Which directions?}}\\
word-local modes\\over time};

\path (d.south) -- (o4.north) coordinate[midway] (hub);
\draw[bus] (d.south) -- (hub);
\fill[cuspnavy] (hub) circle (1.15pt);
\draw[flow] (hub) -| (o1.north);
\draw[flow] (hub) -| (o2.north);
\draw[flow] (hub) -| (o3.north);
\draw[flow] (hub) -- (o4.north);
\end{tikzpicture}
\caption{CUSP turns independent period distributions into one marginal preserving temporal process. Its displacement operators yield compatible LSC readouts of magnitude and timing, mechanism, component and passage attribution, and word-local temporal modes.}
\label{fig:cusp_process_pipeline}
\end{figure*}
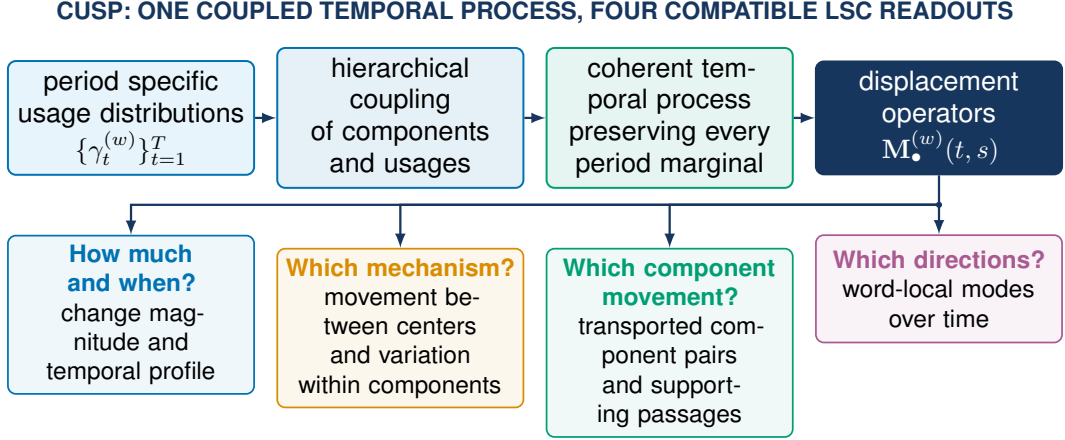

Diachronic corpora provide independent samples from each period's usage distribution, not correspondences between periods. A directly optimized $t{\to}t{+}2$ transport plan can disagree with the plan composed through $t{+}1$. Scores and transition explanations built from those separate plans can therefore describe different histories. What is needed is one process that preserves each period marginal and makes all readouts refer to the same temporal correspondence.

\noindent\textbf{We introduce Coupled Usage--Sense Processes (CUSP) to construct exactly this object (Figure~\ref{fig:cusp_process_pipeline}).} CUSP makes three technical contributions. (i) A hierarchical coupling matches contextual distributions within component pairs, transports their prevalence mass with the prescribed marginals, and composes adjacent plans into coherent longer span correspondence. (ii) Displacement operators quantify magnitude and timing and decompose change exactly by center movement, within component variation, and transported component pair. The summed mean operator yields word-local modes, while representative passages make the attributed movements inspectable. (iii) Building on the second-moment geometry of MENT \citep{ezoe2026multiscale}, we prove marginal preservation, temporal compatibility, exact decompositions, and parametric recovery of the operators and squared distances for the trace and retained modes under identification and optimization stability conditions. These are compatible readouts of one fitted temporal process.

Empirically, we test each level of this construction in increasing proximity to natural LSC. Synthetic data support the predicted recovery rate. DWUG establishes that the operator trace retains standard scalar validity while its exact decomposition distinguishes mechanisms. Janus tests temporal profile recovery and compositional coherence across several periods. A large corpus of U.S.\ court opinions shows that the complete construction connects an unsupervised temporal signal to component movements, word-local modes, and directly inspectable passages. Together, these settings show that CUSP preserves the conventional answer to \emph{how much} while adding compatible answers to \emph{when}, \emph{through which movement}, and \emph{with what textual evidence}.

\section{Coupled Usage--Sense Processes}
\label{sec:coupled_processes}

A fixed encoder places independently sampled usages from each period in a common embedding space, but shared coordinates do not determine temporal correspondence. CUSP represents each period distribution as a mixture of latent usage components and couples their prevalences and contextual distributions into a joint process whose period marginals reproduce the observed distributions.
\subsection{Usage distributions in each period}

Let \(V\) be the vocabulary and \([T]=\{1,\ldots,T\}\) the periods. For word \(w\in V\) and period \(t\in[T]\), let
\[
X_{t,i}^{(w)}\overset{\mathrm{i.i.d.}}{\sim}\gamma_t^{(w)},\qquad
\gamma_t^{(w)}=\sum_{k=1}^{K_t^{(w)}}\pi_{t,k}^{(w)}\nu_{t,k}^{(w)},\qquad
\nu_{t,k}^{(w)}\in\mathcal P_2(\mathbb R^d).
\]
Here, \(X_{t,i}^{(w)}\) is a contextual usage embedding, \(K_t^{(w)}\) counts components, \(\pi_{t,k}^{(w)}\) are positive prevalence weights summing to one, and \(\nu_{t,k}^{(w)}\) is the contextual distribution of component \(k\). Components are period local, so equal indices across periods do not imply correspondence. Mixtures are fitted independently by period, and their marginals do not identify temporal correspondence. After coupling, prevalence changes describe relative mass redistribution, while component changes describe movement or reshaping. 

%%%% BEFORE
%Each sample identifies its period marginal but not temporal correspondence. After coupling, changes in prevalence describe the redistribution of usage mass, while changes in component distributions describe movement or reshaping.
%Each period sample identifies its own marginal distribution, but no observation links a usage at \(t\) to one at \(s\). Changes in \(\pi_{t,k}^{(w)}\) represent redistribution among components, whereas changes in \(\nu_{t,k}^{(w)}\) represent movement or reshaping of their contextual distributions.
%%Here \(X_{t,i}^{(w)}\) is the contextual embedding of usage \(i\), \(\pi_{t,k}^{(w)}>0\) is the prevalence of Sense component \(k\), and \(\nu_{t,k}^{(w)}\) is its contextual distribution. We assume that the mixture is identifiable up to permutation of component labels.

\subsection{Hierarchical coupling and Markov composition}

We use optimal transport hierarchically: contextual distributions are coupled within each candidate component pair, and their induced costs define transport between component prevalences, following the general structure of mixture and hierarchical OT \citep{chen2019optimal,delon2020wasserstein,yurochkin2019hierarchical}. For each adjacent pair \((t,t+1)\) and possible component transition \(k\to\ell\), CUSP first selects a contextual coupling
\begin{equation}
\eta_{t,k,\ell}^{(w)}
=
\arg\min_{\eta\in\mathcal A_{t,k,\ell}^{(w)}}
\int c(x,y)\,d\eta(x,y),
\qquad
\mathcal A_{t,k,\ell}^{(w)}
\subseteq
\Pi\!\left(\nu_{t,k}^{(w)},\nu_{t+1,\ell}^{(w)}\right),
\label{eq:pairwise_distribution_coupling}
\end{equation}
where \(\mathcal A_{t,k,\ell}^{(w)}\) is the admissible class of contextual couplings for the component pair and equals \(\Pi(\nu_{t,k}^{(w)},\nu_{t+1,\ell}^{(w)})\) in the Gaussian specialization below. We denote the minimum cost by \(C_{t,k,\ell}^{(w)}\). CUSP then transports the component prevalences using
\begin{equation}
\begin{aligned}
Q_t^{(w)}
=
\arg\min_{Q}\quad&
\sum_{k,\ell}Q(k,\ell)C_{t,k,\ell}^{(w)}
\\
\text{s.t.}\quad&
\sum_{\ell}Q(k,\ell)=\pi_{t,k}^{(w)},\qquad
\sum_kQ(k,\ell)=\pi_{t+1,\ell}^{(w)},\qquad
Q(k,\ell)\geq0.
\end{aligned}
\label{eq:sense_transition_distribution}
\end{equation}
The two transport levels have distinct roles. \(C_{t,k,\ell}^{(w)}\) measures the contextual cost of a candidate component transition, while \(Q_t^{(w)}\) allocates prevalence mass among those candidates subject to the two observed period marginals.
%%Thus \(Q_t^{(w)}(k,\ell)\) is the mass assigned to transition \(k\to\ell\). We assume that the contextual and Sense-level optimization problems have unique minimizers.

\begin{assumption}[Uniqueness of the coupling construction]
\label{ass:coupling_construction_uniqueness}
For every adjacent period pair, each problem in \eqref{eq:pairwise_distribution_coupling} has a unique finite-cost minimizer, and \eqref{eq:sense_transition_distribution} has a unique minimizer.
\end{assumption}

%%%% BEFORE
%CUSP fixes the corresponding Markov composition so that all non-adjacent correspondences are induced by the same adjacent plans.

The adjacent plans share the prescribed intermediate marginals and can therefore be composed into a joint coupling \citep{peyre2019computational}. CUSP uses Markov composition to induce non-adjacent correspondences from adjacent plans.  Specifically, let \((\Omega,\mathcal F,\mathbb P)\) be a probability space on which the latent component process \(Z_{1:T}^{(w)}\) and the contextual state process \(X_{1:T}^{(w)}\) are defined. The adjacent plans define a Markov chain over usage components local to each period:
\[
\Pr(Z_1^{(w)}=k)=\pi_{1,k}^{(w)},\qquad
P_t^{(w)}(k,\ell)
:=
\Pr(Z_{t+1}^{(w)}=\ell\mid Z_t^{(w)}=k)
=
\frac{Q_t^{(w)}(k,\ell)}{\pi_{t,k}^{(w)}}.
\]
Disintegrating \(\eta_{t,k,\ell}^{(w)}(dx,dy)=\nu_{t,k}^{(w)}(dx)\mathcal K_{t,k,\ell}^{(w)}(x,dy)\) defines the corresponding contextual process conditional on the entire latent component sequence \(Z_{1:T}^{(w)}\):
\[
X_1^{(w)}\mid Z_{1:T}^{(w)}
\sim
\nu_{1,Z_1^{(w)}}^{(w)},
\qquad
X_{t+1}^{(w)}\mid (X_t^{(w)},Z_{1:T}^{(w)})
\sim \mathcal K_{t,Z_t^{(w)},Z_{t+1}^{(w)}}^{(w)}(X_t^{(w)},\cdot).
\]

The construction preserves both levels of the observed mixture representation.
\begin{proposition}[Marginal preservation and temporal compatibility]
\label{prop:marginal_preservation}
For every \(w\in V\), \(t\in[T]\), and \(k\in[K_t^{(w)}]\),
\[
\Pr(Z_t^{(w)}=k)=\pi_{t,k}^{(w)},\qquad
\mathcal L(X_t^{(w)}\mid Z_{1:T}^{(w)})=\nu_{t,Z_t^{(w)}}^{(w)}.
\]
Consequently,
\begin{equation}
\mathcal L(X_1^{(w)},\ldots,X_T^{(w)})
\in
\Pi(\gamma_1^{(w)},\ldots,\gamma_T^{(w)}).
\label{eq:joint_coupling_marginals}
\end{equation}
% For \(s>t\), its non-adjacent Sense plan is induced by the same process:
% \[
% R_{t,s}^{(w)}=P_t^{(w)}P_{t+1}^{(w)}\cdots P_{s-1}^{(w)},\qquad
% q^{(w)}(t,s;k,\ell)=\pi_{t,k}^{(w)}R_{t,s}^{(w)}(k,\ell).
% \]
\end{proposition}
% The proof and the dynamic programming evaluation of these quantities are given in Appendix.
The correspondence from \(t\) to \(s\) is therefore induced through the intervening adjacent plans, rather than being optimized independently. All period pairs belong to one process with the observed marginals, and there is no assertion of trajectories for individual utterances.

\subsection{Full, mean, and deviation processes}

Let \(m_{t,k}^{(w)}=\int x\,d\nu_{t,k}^{(w)}(x)\). The coupled contextual state can be separated into a component center process and a centered deviation process within components:
\[
\phi_{\mathrm{full}}^{(w)}(t)=X_t^{(w)},\qquad
\phi_{\mathrm{mean}}^{(w)}(t)=m_{t,Z_t^{(w)}}^{(w)},\qquad
\phi_{\mathrm{dev}}^{(w)}(t)=X_t^{(w)}-m_{t,Z_t^{(w)}}^{(w)}.
\]

These processes satisfy the following reconstruction and centering properties:
% By marginal preservation,
\begin{equation}
\phi_{\mathrm{full}}^{(w)}(t)
=
\phi_{\mathrm{mean}}^{(w)}(t)+\phi_{\mathrm{dev}}^{(w)}(t),
\qquad
\mathbb E[\phi_{\mathrm{dev}}^{(w)}(t)\mid Z_{1:T}^{(w)}]=0.
\label{eq:full_mean_dev_decomposition}
\end{equation}
A component may therefore move through the embedding space, or its usages may broaden, contract, or rotate around its center. The full process retains both mechanisms.
%%The full process represents total contextual evolution, the mean process movement among transported Sense centers, and the deviation process changes in within-Sense contextual variation.

\section{Quantifying and attributing semantic change}
\label{sec:change_operators}

The coupled process provides compatible answers to how much a word has changed, when it changed, and which mechanisms and component movements carried the change. Second-moment operators retain directional information, their traces yield scalar change scores and temporal profiles, and their exact decompositions attribute change to movement between components, variation within components, and transported component pairs.

\subsection{Displacement operators and temporal profiles}

For \(\bullet\in\{\mathrm{full},\mathrm{mean},\mathrm{dev}\}\), define
\[
\Delta\phi_{\bullet}^{(w)}(t,s)
:=
\phi_{\bullet}^{(w)}(t)-\phi_{\bullet}^{(w)}(s),\qquad
\mathbf M_{\bullet}^{(w)}(t,s)
:=
\mathbb E[\Delta\phi_{\bullet}^{(w)}(t,s)\Delta\phi_{\bullet}^{(w)}(t,s)^\top].
\]
The trace distance
\[
d_{\bullet,\mathrm{tr}}^{(w)}(t,s)
:=
\{\operatorname{tr}\mathbf M_{\bullet}^{(w)}(t,s)\}^{1/2}
\]
measures the total coupled displacement between two periods. Evaluating this distance on adjacent pairs yields a temporal change profile. The trace reflects how much change occurred, while the operator retains the embedding directions carrying that displacement. Adjacent evaluations locate the change in time, and non-adjacent evaluations inherit the same composed temporal correspondence. For any fixed word-local orthonormal basis \(\{\bm u_r^{(w)}\}_{r=1}^d\), define
\[
d_{\bullet,r}^{(w)}(t,s)
:=
\{(\bm u_r^{(w)})^\top\mathbf M_{\bullet}^{(w)}(t,s)\bm u_r^{(w)}\}^{1/2}.
\]

The trace summarizes the total semantic variation, while the mode-wise quantities provide its
orthogonal decomposition:
\begin{equation}
\left(d_{\bullet,\mathrm{tr}}^{(w)}(t,s)\right)^2
=
\sum_{r=1}^d\left(d_{\bullet,r}^{(w)}(t,s)\right)^2.
\label{eq:trace_mode_decomposition}
\end{equation}

%\textbf{Remark.} Since \(\nu_{t,k}^{(w)}\in P_2(\mathbb R^d)\), all \(\phi_{\bullet}^{(w)}(t)\) belong to \(L^2(\Omega;\mathbb R^d)\), so \(\mathbf M_{\bullet}^{(w)}(t,s)\) and the resulting distances are finite. Moreover, the trace and mode-wise semantic distances define finite pseudo-metrics on \([T]\), with the zero-distance conditions given by

\textbf{Remark.} Because \(\nu_{t,k}^{(w)}\in\mathcal P_2(\mathbb R^d)\), all \(\phi_{\bullet}^{(w)}(t)\) belong to \(L^2(\Omega;\mathbb R^d)\). Hence \(\mathbf M_{\bullet}^{(w)}(t,s)\) and the resulting distances are finite. Moreover, the trace and mode-wise semantic distances define finite pseudometrics on \([T]\), with the zero-distance conditions given by \(
\phi_{\bullet}^{(w)}(t)=\phi_{\bullet}^{(w)}(s)
\quad\text{almost surely}
\),
and
\(
(\bm u_r^{(w)})^\top\phi_{\bullet}^{(w)}(t)
=
(\bm u_r^{(w)})^\top\phi_{\bullet}^{(w)}(s)
\quad\text{almost surely}
\),
respectively. These properties are the same as those of the corresponding distances defined in \cite{ezoe2026multiscale}. Appendix~\ref{subsec:semantic_distance_properties} gives the details.

\subsection{Exact mechanism and component transition attribution}

Conditional centering removes the mean--deviation cross term.
\begin{theorem}[Exact mean--deviation decomposition]
\label{thm:exact_mean_deviation_decomposition}
For every \(w\in V\) and \(t,s\in[T]\),
\begin{equation}
\mathbf M_{\mathrm{full}}^{(w)}(t,s)
=
\mathbf M_{\mathrm{mean}}^{(w)}(t,s)
+
\mathbf M_{\mathrm{dev}}^{(w)}(t,s).
\label{eq:mean_dev_operator_decomposition}
\end{equation}
Consequently, for \(\rho\in\{\mathrm{tr},1,\ldots,d\}\),
\begin{equation}
\left(d_{\mathrm{full},\rho}^{(w)}(t,s)\right)^2
=
\left(d_{\mathrm{mean},\rho}^{(w)}(t,s)\right)^2
+
\left(d_{\mathrm{dev},\rho}^{(w)}(t,s)\right)^2.
\label{eq:mean_dev_distance_decomposition}
\end{equation}
\end{theorem}

Further conditioning on \((Z_t^{(w)},Z_s^{(w)})\) attributes each operator to transported component pairs. Define
\[
\begin{aligned}
\mathbf C_{\mathrm{mean}}^{(w)}(t,s;k,\ell)
&:=
(m_{t,k}^{(w)}-m_{s,\ell}^{(w)})(m_{t,k}^{(w)}-m_{s,\ell}^{(w)})^\top,\\
\mathbf C_{\mathrm{dev}}^{(w)}(t,s;k,\ell)
&:=
\operatorname{Cov}(X_t^{(w)}-X_s^{(w)}\mid Z_t^{(w)}=k,Z_s^{(w)}=\ell),\\
q^{(w)}(t,s;k,\ell)
&:=
\Pr
\left(
Z_t^{(w)}=k,\,
Z_s^{(w)}=\ell
\right).
\end{aligned}
\]
The component-pair covariance \(\mathbf C_{\mathrm{dev}}^{(w)}(t,s;k,\ell)\) is defined only when \(q^{(w)}(t,s;k,\ell)>0\). Zero-mass pairs contribute nothing and are omitted from the weighted sums.
% \mathbf C_{\mathrm{full}}^{(w)}(t,s;k,\ell) &:= \mathbf C_{\mathrm{mean}}^{(w)}(t,s;k,\ell)+\mathbf C_{\mathrm{dev}}^{(w)}(t,s;k,\ell).

\begin{theorem}[Exact component transition attribution]
\label{thm:exact_sense_transition_attribution}
For \(\bullet\in\{\mathrm{mean},\mathrm{dev}\}\),
% For \(\bullet\in\{\mathrm{full},\mathrm{mean},\mathrm{dev}\}\),
\begin{equation}
\mathbf M_{\bullet}^{(w)}(t,s)
=
\sum_{k=1}^{K_t^{(w)}}\sum_{\ell=1}^{K_s^{(w)}}
q^{(w)}(t,s;k,\ell)\mathbf C_{\bullet}^{(w)}(t,s;k,\ell).
\label{eq:sense_transition_operator_decomposition}
\end{equation}
Consequently,
\begin{align}
\left(d_{\bullet,\mathrm{tr}}^{(w)}(t,s)\right)^2
&=
\sum_{k,\ell}q^{(w)}(t,s;k,\ell)\operatorname{tr}\mathbf C_{\bullet}^{(w)}(t,s;k,\ell),
\label{eq:trace_sense_transition}\\
\left(d_{\bullet,r}^{(w)}(t,s)\right)^2
&=
\sum_{k,\ell}q^{(w)}(t,s;k,\ell)
(\bm u_r^{(w)})^\top\mathbf C_{\bullet}^{(w)}(t,s;k,\ell)\bm u_r^{(w)}.
\label{eq:mode_sense_transition}
\end{align}
\end{theorem}
% These are exact identities for the fitted CUSP process: they connect a total or mode-wise score to its between-Sense, within-Sense, and transported-transition contributions.
%%These are exact identities for the fitted CUSP process: they connect a trace or mode-wise score to its between-Sense and within-Sense components, weighted by transported-transition contributions.
Together, Theorems~\ref{thm:exact_mean_deviation_decomposition} and~\ref{thm:exact_sense_transition_attribution} form an exact accounting of change. Every unit of the squared trace or mode-wise variation is assigned to between center or within component change and to transported component pairs. This attribution decomposes the score itself rather than fitting a separate explanatory model. In the empirical analysis, passages assigned to the selected source and target components make these numerical attributions inspectable.

\subsection{Word-local temporal modes}

To adapt the basis to the history of word \(w\), we use the aggregated mean operator
\[
\mathcal M_{\mathrm{mean}}^{(w)}
:=
\sum_{t=1}^{T-1}\mathbf M_{\mathrm{mean}}^{(w)}(t,t+1).
\]
We take the orthonormal eigenvectors of \(\mathcal M_{\mathrm{mean}}^{(w)}\) as the word specific basis \(\{\bm u_r^{(w)}\}_{r=1}^d\), with eigenvalues ordered decreasingly:
\[
\mathcal M_{\mathrm{mean}}^{(w)}\bm u_r^{(w)}
=
\lambda_r^{(w)}\bm u_r^{(w)}.
\]

Let \(R^{(w)}\in[d]\) denote the number of modes retained for individual analysis. We require their eigenvalues to be non-degenerate\footnote{For \(r>R^{(w)}\), the eigenvectors need not be uniquely determined in the presence of eigenvalue degeneracy, but this freedom is immaterial because these modes are not used for individual analysis.}.

\begin{assumption}[Non-degenerate retained modes]
\label{assump:separated_eigenvalues}
There exists a constant \(\delta>0\) such that, for every \(r\in[R^{(w)}]\),
%%%% BEFORE %%%%
%\[
%\lambda_{r-1}^{(w)}-\lambda_r^{(w)}\ge \delta,
%\]
\[
\lambda_r^{(w)}-\lambda_{r+1}^{(w)}\ge \delta,
\]
where we define \(\lambda_{d+1}^{(w)}=-\infty\).
\end{assumption}
%\begin{assumption}[Separated retained modes]
%\label{assump:separated_eigenvalues}
%There exists \(\delta>0\) such that every retained eigenvalue is separated from the rest of the spectrum:
%\[
%\min_{j\in[d]\setminus\{r\}}
%\left|\lambda_r^{(w)}-\lambda_j^{(w)}\right|
%\geq\delta
%\qquad\text{for every }r\in[R^{(w)}].
%\]
%\end{assumption}

The eigenpairs admit the following spectral characterization.
\begin{proposition}[Spectral characterization of word-local modes]
\label{prop:word_local_modes}
For every \(r\in[d]\),
\begin{equation}
\lambda_r^{(w)}
=
\sum_{t=1}^{T-1}
\left(d_{\mathrm{mean},r}^{(w)}(t,t+1)\right)^2,
\label{eq:eigenvalue_accounting}
\end{equation}
and
\begin{equation}
\bm u_r^{(w)}
\in
\arg\max_{\substack{\|\bm u\|=1\\
\bm u\perp\bm u_1^{(w)},\ldots,\bm u_{r-1}^{(w)}}}
\sum_{t=1}^{T-1}
\bm u^\top\mathbf M_{\mathrm{mean}}^{(w)}(t,t+1)\bm u.
\label{eq:word_local_variational_characterization}
\end{equation}
\end{proposition}
The path sum identifies the directions that organize the word's complete history, while the adjacent distances locate each direction in time. Because the basis is word local, its modes describe distinct developments within a word rather than vocabulary wide axes.
% More generally, the same construction may use \(\mathcal M_{\bullet}^{(w)}=\sum_t\mathbf M_{\bullet}^{(w)}(t,t+1)\) for \(\bullet\in\{\mathrm{full},\mathrm{mean},\mathrm{dev}\}\); the spectral identity and recovery argument are unchanged.
We use the mean basis to isolate center movement while keeping the full, mean, and deviation contributions exactly comparable.

\subsection{Gaussian implementation and statistical recovery}

For a tractable estimator, we specialize to Gaussian usage components, quadratic cost, and unrestricted component couplings using the established Gaussian and Gaussian-mixture Wasserstein geometry \citep{chen2019optimal,delon2020wasserstein}:
\[
\nu_{t,k}^{(w)}=\mathcal N(\mu_{t,k}^{(w)},\Sigma_{t,k}^{(w)}),\qquad
\mathcal A_{t,k,\ell}^{(w)}=\Pi(\nu_{t,k}^{(w)},\nu_{t+1,\ell}^{(w)}),\qquad
c(x,y)=\|x-y\|^2.
\]

For completeness, the resulting component-level optimal coupling has the standard Gaussian form.
\begin{proposition}[Gaussian optimal coupling]
\label{prop:gaussian_optimal_coupling}
If \(\Sigma_{t,k}^{(w)}\succ0\), the optimizer in \eqref{eq:pairwise_distribution_coupling} is unique and equals
\begin{equation}
\eta_{t,k,\ell}^{(w)}
=
(\mathrm{id},T_{t,k,\ell}^{(w)})_\#\nu_{t,k}^{(w)},\qquad
T_{t,k,\ell}^{(w)}(x)
=
\mu_{t+1,\ell}^{(w)}
+
A_{t,k,\ell}^{(w)}(x-\mu_{t,k}^{(w)}),
\label{eq:gaussian_optimal_coupling}
\end{equation}
where
\[
A_{t,k,\ell}^{(w)}
=
\left(\Sigma_{t,k}^{(w)}\right)^{-1/2}
\Bigl[
\left(\Sigma_{t,k}^{(w)}\right)^{1/2}
\Sigma_{t+1,\ell}^{(w)}
\left(\Sigma_{t,k}^{(w)}\right)^{1/2}
\Bigr]^{1/2}
\left(\Sigma_{t,k}^{(w)}\right)^{-1/2}.
\]
\end{proposition}
Under this specialization, the component transition terms and induced distances admit closed form expressions and can be computed by dynamic programming for every period pair \((t,s)\). Details are deferred to Appendices~\ref{app:gaussian_expressions} and~\ref{app:dynamic_programming_transition}.

% Let \(n^{(w)}=\min_{t}n_t^{(w)}\to\infty\)
Finally, we have the statistical recovery results. Let hats denote estimated quantities.
% and define \[E_{M,\bullet}^{(w)}:=\max_{t,s}\|\widehat{\mathbf M}_{\bullet}^{(w)}(t,s)-\mathbf M_{\bullet}^{(w)}(t,s)\|_2.\]

%\begin{assumption}[Estimation and optimization conditions]
%\label{ass:gaussian_estimation_regularity}
%The Gaussian mixtures are identifiable, their component numbers are known, their covariances are positive definite, and their parameters are estimated by penalized maximum likelihood procedure of\cite{chen2009inference}. Each optimal component plan is unique and has support size \(K_t^{(w)}+K_{t+1}^{(w)}-1\).
%\end{assumption}
\begin{assumption}[Estimation and optimization conditions]
\label{ass:gaussian_estimation_regularity}
The Gaussian mixtures have known, correctly specified component counts, positive weights, positive definite covariances and pairwise distinct component parameters (see Definition~\ref{def:gmm_class}). Their parameters are estimated by a penalized maximum likelihood estimator satisfying the conditions of \citet{chen2009inference}. Each adjacent optimal component plan is unique and has support size \(K_t^{(w)}+K_{t+1}^{(w)}-1\).
\end{assumption}

%Because \(n^{(w)}:=\min_t n_t^{(w)}\to\infty\), for
\begin{theorem}[Statistical recovery]
\label{thm:distance_recovery}
Under the Gaussian specialization and Assumptions~\ref{assump:separated_eigenvalues} and~\ref{ass:gaussian_estimation_regularity}, the following holds using the word-local basis induced by \(\mathcal M_{\mathrm{mean}}^{(w)}\). As \(n^{(w)}:=\min_t n_t^{(w)}\to\infty\), for \(\bullet\in\{\mathrm{full},\mathrm{mean},\mathrm{dev}\}\) and \(\rho\in\{\mathrm{tr},1,\ldots,R^{(w)}\}\),
\[
\max_{t,s}
\left\|
\widehat{\mathbf M}_{\bullet}^{(w)}(t,s)
-
\mathbf M_{\bullet}^{(w)}(t,s)
\right\|_2
=
O_p\!\left((n^{(w)})^{-1/2}\right),
\]
and
\[
\max_{t,s}\left|
\widehat d_{\bullet,\rho}^{(w)}(t,s)^2
-
d_{\bullet,\rho}^{(w)}(t,s)^2
\right|
=
O_p\!\left((n^{(w)})^{-1/2}\right).
\]
\end{theorem}
Thus, the operators, squared trace and mode distances used below stabilize jointly as the period sample sizes increase. The guarantee covers the CUSP quantities derived from the fitted process, as well as the underlying Gaussian mixture parameters.

% \begin{theorem}[Statistical recovery]
% \label{thm:distance_recovery}
% Under the Gaussian specialization and Assumptions~\ref{assump:separated_eigenvalues} and \ref{ass:gaussian_estimation_regularity}, for \(\bullet\in\{\mathrm{full},\mathrm{mean},\mathrm{dev}\}\) and \(\rho\in\{\mathrm{tr},1,\ldots,d\}\),
% \[
% E_{M,\bullet}^{(w)}=O_p(n^{-1/2}),\qquad \left|\widehat d_{\bullet,\rho}^{(w)}(t,s)^2-d_{\bullet,\rho}^{(w)}(t,s)^2\right|=O_p(n^{-1/2}).
% \]
% Moreover,
% \[
% \|\widehat{\mathcal M}_{\mathrm{mean}}^{(w)}-\mathcal M_{\mathrm{mean}}^{(w)}\|_2=O_p(n^{-1/2}).
% \]
% If \(\lambda_r^{(w)}\) is separated from the remaining spectrum, then, after sign alignment,
% \[
% \|\widehat{\bm u}_r^{(w)}-\bm u_r^{(w)}\|_2=O_p(n^{-1/2}),\qquad
% \left|\widehat d_{\bullet,r}^{(w)}(t,s)^2-d_{\bullet,r}^{(w)}(t,s)^2\right|=O_p(n^{-1/2}).
% \]
% \end{theorem}

\section{Experiments}
\label{sec:experiments}

The experiments follow Figure~\ref{fig:cusp_process_pipeline}: synthetic data test operator and distance recovery, DWUG \citep{schlechtweg2021dwug} tests scalar validity and decomposition, Janus \citep{cassotti2025sense} tests multi-period recovery and coherence, and Court opinions test passage-grounded attribution. The synthetic experiment uses Gaussian samples. The lexical experiments use target-token XL-LEXEME embeddings \citep{cassotti2023xllexeme} and Gaussian CUSP, with component couplings given by Proposition~\ref{prop:gaussian_optimal_coupling}. The dominant directions shared across contextual embeddings can distort the distances \citep{timkey2021rogue}. For each dataset, we therefore center its pooled usage panel, remove the leading principal direction ($k=1$), and apply $\ell_2$-normalization \citep{mu2018all}. The transform is fixed across words and periods within each dataset. DWUG uses separate English and German panels, Janus uses its original released pool before schedule construction, and Court uses its balanced analysis panel.

%The synthetic experiment uses Gaussian samples. The lexical experiments use target-token XL-LEXEME embeddings \citep{cassotti2023xllexeme} and Gaussian CUSP, with component couplings given by Proposition~\ref{prop:gaussian_optimal_coupling}. The dominant directions shared across contextual embeddings can distort the distances \citep{timkey2021rogue}. For each dataset, we therefore center its pooled usage panel, remove the leading principal direction ($k=1$), and apply $\ell_2$-normalization \citep{mu2018all}. The transform is fixed across words and periods within each dataset.  DWUG uses separate English and German panels, Janus uses its original released pool before schedule construction, and Court uses its balanced analysis panel. 

%%%% BEFORE
%DWUG uses separate English and German panels, Janus uses its original released pool before schedule construction, and Court uses its balanced analysis panel.

\subsection{Synthetic mixture processes: finite-sample recovery}
\label{sec:experiments_synthetic}

We test Theorem~\ref{thm:distance_recovery} on a three-period Gaussian process in $d=2$ with $K=2$, where the component prevalences, centers, and covariances all change. For $50$ repetitions at $n\in\{100,200,400,800,1600\}$ usages per period, we fit independent full-covariance Gaussian mixture models (GMMs) with $K$ fixed and withhold the planted component labels.

\begin{figure*}[t]
\centering
\includegraphics[width=0.315\textwidth]{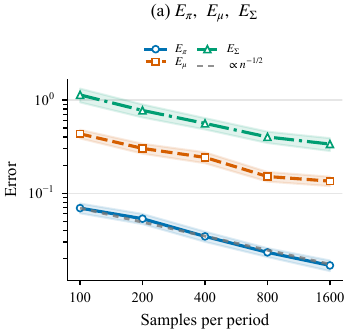}\hfill
\includegraphics[width=0.315\textwidth]{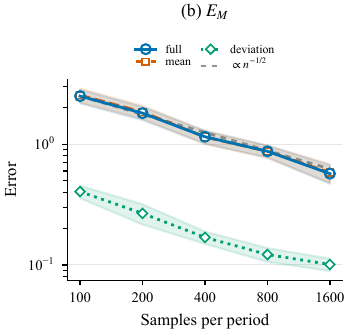}\hfill
\includegraphics[width=0.315\textwidth]{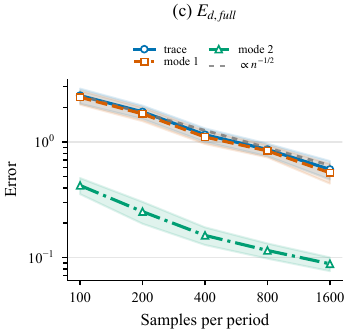}
\caption{Finite-sample recovery of (a) component parameters, (b) full, mean, and deviation operators, and (c) squared full-trace and word-local mode distances. Curves show means and $95\%$ confidence intervals over $50$ repetitions. Gray dashed lines show the $n^{-1/2}$ rate.}
\label{fig:synthetic_recovery}
\end{figure*}
Figure~\ref{fig:synthetic_recovery} shows that the parameter errors decrease with increasing sample size. The full operator and squared full-trace errors both have a log--log slope of $-0.532$, while the two squared mode distance errors have slopes of $-0.540$ and $-0.561$. Their confidence intervals contain $-1/2$, agreeing with the rate predicted by Theorem~\ref{thm:distance_recovery}. Appendix~\ref{app:synthetic_details} reports the mean and deviation distances and pure-mechanism checks.

\subsection{DWUG: scalar validity and decomposition}
\label{sec:experiments_dwug}

The English and German DWUG datasets contain $46$ and $50$ two-period targets with human graded-change scores \citep{schlechtweg2021dwug}. We compare the CUSP full trace with APD, prototype distance, and AP-JSD \citep{giulianelli2020analysing,periti2024systematic}, balanced usage-level OT \citep{montariol2021transport}, and fSUS \citep{kishino2025quantifying}. We also evaluate the exact mean--deviation split (Theorem~\ref{thm:exact_mean_deviation_decomposition}). Tables~\ref{tab:dwug_results} and~\ref{tab:app_dwug_decomposition} report the full/mean and complete results, respectively.

\begin{figure*}[t]
\begin{minipage}[t]{0.55\textwidth}
\vspace{0pt}
\centering
\captionof{table}{Spearman correlation with DWUG judgments under pooled panel $k=1$. Brackets give $95\%$ paired word-bootstrap intervals.}
\label{tab:dwug_results}
\setlength{\tabcolsep}{2.5pt}
\resizebox{\linewidth}{!}{%
\begin{tabular}{lcc}
\toprule
Method & English $\rho$ [95\% CI] & German $\rho$ [95\% CI] \\
\midrule
\textbf{CUSP full trace} & \textbf{$0.746$ [$0.536,0.883$]} & $0.815$ [$0.664,0.903$] \\
CUSP mean trace & $0.742$ [$0.531,0.881$] & $0.809$ [$0.649,0.902$] \\
APD & $0.706$ [$0.482,0.853$] & $0.815$ [$0.669,0.894$] \\
Prototype & $0.636$ [$0.375,0.817$] & $0.801$ [$0.641,0.896$] \\
AP-JSD & $0.531$ [$0.242,0.737$] & $0.578$ [$0.344,0.761$] \\
Usage-level OT & $0.718$ [$0.501,0.865$] & \textbf{$0.827$ [$0.675,0.915$]} \\
fSUS & $0.726$ [$0.516,0.866$] & $0.782$ [$0.621,0.881$] \\
\bottomrule
\end{tabular}%
}
\end{minipage}\hfill
\begin{minipage}[t]{0.43\textwidth}
\vspace{0pt}
\centering
\includegraphics[width=\linewidth]{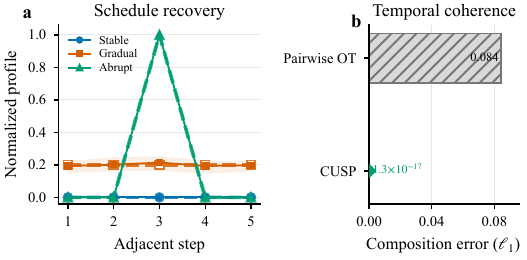}
\captionsetup{font=footnotesize,skip=3pt}
%%%% BEFORE
%\captionof{figure}{Janus at pooled panel $k=1$. (a) Planted and recovered profiles. Bands span the $2.5$--$97.5$ percentiles across $16$ lemmas per schedule. (b) Mean $\ell_1$ composition error for pairwise OT and CUSP.}
\captionof{figure}{Janus at $k=1$. (a) Planted and recovered profiles. Bands span the $2.5$--$97.5$ percentiles across $16$ lemmas per schedule. (b) Mean $\ell_1$ composition error. CUSP is consistent by construction.}
\label{fig:janus_recovery}
\end{minipage}
\end{figure*}

Table~\ref{tab:dwug_results} places CUSP first in English and within $0.012$ of the strongest German result. The overlapping intervals support competitive scalar validity rather than uniform superiority. Sensitivity at $k=0$ and $k=3$ preserves the same conclusion, as reported in Appendix~\ref{app:dwug_details}. Center movement supplies median shares of $0.868$ and $0.906$ of the full-trace variation in English and German. The deviation trace remains correlated with human judgments at $0.681$ and $0.674$. The deviation distance is the second largest in each dataset for English \emph{bar} ($0.323$) and German \emph{Rezeption} ($0.326$), both split across multiple DWUG usage clusters. The decomposition therefore preserves a useful within-component reorganization signal even though center displacement drives most graded change.

\subsection{Janus: controlled multi-period recovery}
\label{sec:experiments_janus}

Beyond two-period DWUG, Janus \citep{cassotti2025sense} provides controlled six-period histories: $48$ lemmas, each with four ten-sentence pools resampled into $40$ usages per period. Stable, gradual, and abrupt prevalence schedules each receive $16$ lemmas. Released sense tags define schedules but are withheld during fitting. CUSP receives no cross-period usage links. A BIC-selected diagonal GMM is fitted independently to every lemma and period. Figure~\ref{fig:janus_recovery}(a) shows that CUSP achieves exact recovery of the stable profile, a TV error of $0.044$ for the gradual profile, and the correct boundary for all $16$ abrupt lemmas. The recovered and planted adjacent changes correlate at $\rho=0.954$. The path length separates the changed and the stable schedules with an AUROC value of $1.000$. Pairwise OT fits balanced usage-level transport independently for each period pair. Its direct and composed plans have mean $\ell_1$ disagreement $0.084$, concentrated in gradual histories (Figure~\ref{fig:janus_recovery}(b)). CUSP composes non-adjacent plans by construction, giving numerical zero disagreement while preserving endpoint marginals (Proposition~\ref{prop:marginal_preservation}). Appendix~\ref{app:janus_details} gives the protocol and schedule-specific results.

\subsection{Court opinions: passage-grounded attribution}
\label{sec:experiments_court}

%%%% BEFORE
%Court opinions represent a natural, unlabeled setting. We use the May 6, 2024 CourtListener bulk release maintained by the Free Law Project \citep{free_law_project2024courtlistener}. The analysis contains $258{,}480$ balanced usages of $60$ legal lemmas from $100{,}473$ opinions covering eight decade bins from $1950$ through $2020$. The lemmas were selected from law textbook tables of contents and fixed before fitting. CourtListener opinion IDs make every excerpt traceable. A fixed LLM procedure applied identically to every lemma first retains occurrences in which the court interprets or applies the term, then extracts the supporting passage verbatim. Within each lemma, every decade is subsampled to the same count, preventing differences in period sample size alone from driving the change signal. BIC selects a period-local diagonal GMM with $K\leq5$. Appendix~\ref{app:court_data} gives the corpus construction, counts, and fitting details.

Court opinions provide a natural, unlabeled setting: $258{,}480$ balanced usages of $60$ legal lemmas from $100{,}473$ opinions in eight decade bins ($1950$--$2020$), drawn from the May 6, 2024 CourtListener release \citep{free_law_project2024courtlistener}. The vocabulary is fixed before fitting. A fixed LLM procedure retains legal-reasoning usages and extracts verbatim passages. Each excerpt remains traceable through its CourtListener opinion ID. Counts are equalized across decades within each lemma, and BIC selects period-local diagonal GMMs with $K\leq5$. Appendix~\ref{app:court_data} details construction and fitting.

%%%% BEFORE
%The Court dataset provides no gold change points or sense labels, so selection must be numerical and completed before any passage is read. We first locate changes relative to each lemma's own history. Let $d_t$ denote one of its seven adjacent full-trace distances. We standardize the sequence as $z_t=(d_t-\bar d)/\operatorname{sd}(d)$ and select its maximum. This removes differences in absolute scale across lemmas. This provides a descriptive temporal screen rather than a significance test. Requiring $\max_t z_t\geq1.2$ results in $53$ lemmas being retained. At that selected transition, Theorems~\ref{thm:exact_mean_deviation_decomposition} and~\ref{thm:exact_sense_transition_attribution} assign the squared change to transported component pairs through
Without gold change points or sense labels, we select transitions numerically before passage inspection. For each lemma's seven adjacent full-trace distances $d_t$, we compute $z_t=(d_t-\bar d)/\operatorname{sd}(d)$ and select the maximum. The descriptive threshold $\max_t z_t\geq1.2$ retains $53$ lemmas. At the selected transition, Theorems~\ref{thm:exact_mean_deviation_decomposition} and~\ref{thm:exact_sense_transition_attribution} give
\[
a_{k\ell}=q_{k\ell}\operatorname{tr}\!\left(\mathbf C_{\mathrm{mean}}(k,\ell)+\mathbf C_{\mathrm{dev}}(k,\ell)\right),
\qquad d_t^2=\sum_{k,\ell}a_{k\ell}.
\]
With $\alpha_{k\ell}=a_{k\ell}/d_t^2$, the ratio $\alpha_{k\ell}/q_{k\ell}$ compares the share of change in the pair with its share of transported mass. We report $E_{\max}$, the largest such ratio at the selected transition, to identify a movement that is unusually consequential for its size. Component labels are local to each period, so their numerical indices have no semantic meaning. Proposition~\ref{prop:word_local_modes} supplies the complementary temporal analysis. Its word-local mean modes distinguish directions that peak together from directions active in different decades. Appendix~\ref{app:court_audits} presents the fixed thresholds, the top ten transition and temporal-separation rankings, and the passage analyses.

\begin{figure*}[t]
\centering
\includegraphics[width=0.98\textwidth]{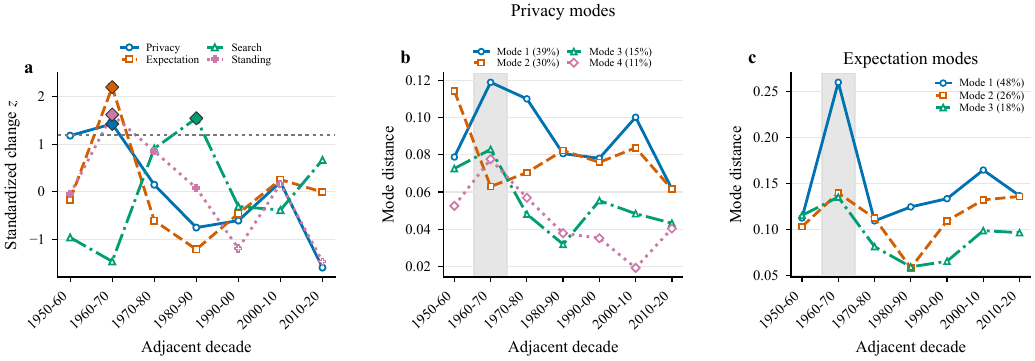}
\caption{Court localization and word-local modes. (a) Standardized adjacent change profiles for the three highest-enrichment lemmas and \emph{expectation}. Diamonds mark maxima and the line is $z=1.2$. (b) \emph{Privacy} modes peak in different decades. (c) The three leading \emph{expectation} modes peak together. Legends give the retained mean-path energy shares. Markers and line styles distinguish series.}
\label{fig:court_temporal_modes}
\end{figure*}
%%%% BEFORE
%Passage retrieval follows the numerical analysis. After the lemma, periods, and attributed component movement or mode are fixed, each usage is assigned to the component with the highest posterior probability. Within the selected source and target components, a fixed score favors usages near the component center with a clear match to the target lemma. Because CUSP models the correspondence between period distributions rather than observed sentence pairs, the source and target examples are retrieved independently as representative endpoints of the attributed movement. Figure~\ref{fig:court_temporal_modes}(a) displays \emph{privacy}, \emph{search}, and \emph{standing}, the three highest $E_{\max}$ lemmas, together with \emph{expectation} as a synchronized contrast. The decisions cited below are historical anchors for the retrieved language, rather than the claimed causes of a corpus-level peak.
After numerical selection, usages are assigned by maximum posterior probability. A fixed retrieval score favors proximity to component centers and clear target-lemma matches (Appendix~\ref{app:court_audits}). Source and target passages are representative endpoints, not observed sentence pairs. Figure~\ref{fig:court_temporal_modes}(a) shows \emph{privacy}, \emph{search}, and \emph{standing}, the three lemmas with the highest $E_{\max}$ scores, alongside \emph{expectation}, whose modes peak together. Cited decisions situate the attributed usage movements in their doctrinal context.

\textbf{Privacy.} \emph{Privacy} ranks first in both Court rankings. Its standardized profile peaks at $1960{\rightarrow}1970$ ($z=1.43$), where center displacement supplies $93.3\%$ of the variation. The selected pair carries $2.46\%$ of the transported mass, but $37.1\%$ of the contribution ($E_{\max}=15.10$). Its source invokes Prosser's ``four categories'' (opinion~\texttt{1179312}) and its target says that warrantless inspection poses ``only a minimal threat to justifiable expectations of privacy'' (opinion~\texttt{1602980}). At the preceding step, the leading pair moves from privacy as ``a direct wrong of a personal character'' (opinion~\texttt{2604478}) toward the same Prosser component. The successive pairs place Prosser's tort classification \citep{prosser1960privacy} between an earlier privacy tort rationale and post-\emph{Katz} expectation-of-privacy language \citep{katz1967unitedstates}. The modes of this lemma separately distinguish tort, constitutional, informational, search, and family privacy across several decades, consistent with legal scholarship treating privacy as context-dependent \citep{solove2006taxonomy}. The underlying passages and case anchors appear in Appendix~\ref{app:court_audits}.

\textbf{Expectation.} Its standardized profile peaks at $1960{\rightarrow}1970$ ($z=2.20$), where center displacement supplies $86.8\%$ of the variation. Sources concern ``an expectation of pay or compensation'' and contributions made ``without expectation of services or direct benefits'' (opinions~\texttt{2226138}, \texttt{1636296}). Targets ask whether the accused showed ``a reasonable expectation of privacy infringed by the search and seizure'' and describe an area ``protected by an expectation of privacy'' (opinions~\texttt{1365600}, \texttt{1311660}). The transition moves from ordinary anticipation toward the post-\emph{Katz} Fourth Amendment standard \citep{katz1967unitedstates,rakas1978illinois}. The three leading modes of this lemma, carrying $48\%$, $26\%$, and $18\%$ of retained mean-path energy, all peak at this transition. Unlike \emph{privacy}, the change is synchronized rather than staggered.  The Court dataset thus demonstrates the complete empirical chain from independently sampled period distributions to a coupled process, exact contributions, temporal modes, and passages that can be checked independently.

\section{Conclusion}
\label{sec:conclusion}
Diachronic corpora provide period specific usage distributions without observed trajectories. CUSP turns them into one marginal preserving process whose adjacent couplings determine longer span correspondence. Its central contribution is a common process yielding magnitude and timing, exact accounting by mechanism and transported component pair, word-local modes, and representative passages. Gaussian specialization makes the operators and squared distances tractable and recoverable. Synthetic, DWUG, Janus, and Court experiments support recovery, scalar validity, multi-period coherence, and passage-grounded attribution. CUSP thus makes scalar measurement, coherent temporal structure, and inspectable explanation compatible views of one lexical history.

\section{Acknowledgements}
We thank Ryoma Kondo, Hiroaki Yamada, Tilmann Altwicker and Zhivko Taushanov for helpful discussions. R.H. is supported by JST FOREST Program (JPMJFR216Q), JST PRESTO Program (JPMJPR2469), Grant-in-Aid for Scientific Research (KAKENHI, JP24K03043), JSPS International Joint Research Program (JRPs with SNSF: 20251501), and the UTEC-UTokyo FSI Research Grant Program.

\bibliographystyle{unsrtnat}
\bibliography{references}
\appendix

\section{Appendix guide and notation}
\label{appx:notation}

The appendix follows the construction in the main text. Appendix~\ref{subsec:sense_mixture_identifiability} formalizes identifiability of the period specific mixture representation. Appendix~\ref{app:basic_process_properties} proves marginal preservation and the basic properties of the coupled process. Appendices~\ref{app:decomposition_properties} and \ref{appx:geometric_characterizations} establish the exact mean--deviation and sense-transition decompositions and the trace and mode-wise geometry, respectively. Appendix~\ref{app:gaussian_expressions} gives explicit Gaussian formulas, and Appendix~\ref{app:dynamic_programming_transition} shows how to compute non-adjacent transition quantities without enumerating latent paths. Appendix~\ref{appx:statistical_recovery} proves the finite-sample recovery result. Appendix~\ref{app:experiments} then provides the complete experimental constructions, robustness checks, rankings, and passage analyses.

Throughout the paper, \(w\) denotes a target word, \(t,s\) denote periods, \(k,\ell\) denote period local components, and \(r\) denotes a word-local mode. Component numbers are local to each period and do not identify the same sense across time. Their relation across periods is determined by the fitted coupling. We use \(\bullet\in\{\mathrm{full},\mathrm{mean},\mathrm{dev}\}\) for statements that apply to any of the three processes and \(\rho\in\{\mathrm{tr},1,\ldots,d\}\) for either the trace or a mode-wise readout. Table~\ref{tab:notation} collects the notation used throughout the paper and appendix.

\begin{table*}[!htbp]
\centering
\caption{Notation used in the CUSP construction, theoretical results, and experiments.}
\label{tab:notation}
\footnotesize
\setlength{\tabcolsep}{5pt}
\renewcommand{\arraystretch}{1.06}
\begin{tabular}{@{}p{0.235\textwidth}p{0.725\textwidth}@{}}
\toprule
\textbf{Notation} & \textbf{Meaning} \\
\midrule
\multicolumn{2}{@{}l}{\textit{Indices, spaces, and period distributions}} \\
\(V,[T],d\) & Vocabulary, period index set \([T]=\{1,\ldots,T\}\), and contextual-embedding dimension. \\
\(w,t,s,k,\ell,r\) & Word, periods, period local source and target components, and word-local mode. \\
\(\mathcal P_2(\mathbb R^d),\Pi(\mu,\nu)\) & Probability measures with finite second moment, and couplings with marginals \(\mu\) and \(\nu\). \\
\(X_{t,i}^{(w)}\sim\gamma_t^{(w)}\) & Contextual embedding of usage \(i\), sampled independently from the distribution of word \(w\) in period \(t\). \\
\(\gamma_t^{(w)}=\sum_k\pi_{t,k}^{(w)}\nu_{t,k}^{(w)}\) & period specific usage distribution represented as a finite mixture. \\
\(K_t^{(w)},\pi_{t,k}^{(w)},\nu_{t,k}^{(w)}\) & Number, prevalence, and contextual distribution of the period local sense components. \\
\(m_{t,k}^{(w)}\) & Center \(\int x\,d\nu_{t,k}^{(w)}(x)\) of a general sense component. \\
\midrule
\multicolumn{2}{@{}l}{\textit{Hierarchical coupling and temporal process}} \\
\(c,\mathcal A_{t,k,\ell}^{(w)}\) & Contextual transport cost and admissible couplings for the candidate transition \(k\to\ell\). \\
\(\eta_{t,k,\ell}^{(w)},C_{t,k,\ell}^{(w)}\) & Optimal adjacent contextual coupling and its scalar transport cost. \\
\(Q_t^{(w)}(k,\ell)\) & Adjacent sense-level transport mass from component \(k\) at \(t\) to component \(\ell\) at \(t+1\). \\
\(P_t^{(w)}(k,\ell)\) & Sense transition probability \(Q_t^{(w)}(k,\ell)/\pi_{t,k}^{(w)}\). \\
\(\mathcal K_{t,k,\ell}^{(w)}(x,dy)\) & Contextual transition kernel obtained by disintegrating \(\eta_{t,k,\ell}^{(w)}\). \\
\(Z_{1:T}^{(w)},X_{1:T}^{(w)}\) & Coupled latent sense process and contextual state process. \\
\(R_{t,s}^{(w)}(k,\ell)\) & Non-adjacent conditional probability \(\Pr(Z_s^{(w)}=\ell\mid Z_t^{(w)}=k)\). \\
\(q^{(w)}(t,s;k,\ell)\) & Joint transported mass \(\Pr(Z_t^{(w)}=k,Z_s^{(w)}=\ell)=\pi_{t,k}^{(w)}R_{t,s}^{(w)}(k,\ell)\). \\
\midrule
\multicolumn{2}{@{}l}{\textit{Processes, operators, decompositions, and modes}} \\
\(\phi_{\mathrm{full}}^{(w)},\phi_{\mathrm{mean}}^{(w)},\phi_{\mathrm{dev}}^{(w)}\) & Contextual state, transported sense center, and centered within-sense deviation processes. \\
\(\Delta\phi_{\bullet}^{(w)}(t,s)\) & Coupled displacement \(\phi_{\bullet}^{(w)}(t)-\phi_{\bullet}^{(w)}(s)\). \\
\(\mathbf M_{\bullet}^{(w)}(t,s)\) & Second-moment operator of \(\Delta\phi_{\bullet}^{(w)}(t,s)\). \\
\(\mathbf C_{\bullet}^{(w)}\) &
For \(\bullet\in\{\mathrm{mean},\mathrm{dev}\}\), \(\mathbf C_{\bullet}^{(w)}(t,s;k,\ell)\) is the operator contribution conditional on transported component pair \((k,\ell)\). \\
\(\bullet\in\{\mathrm{full},\mathrm{mean},\mathrm{dev}\}\) & Choice of the full contextual, between center, or within component process. \\
\(\rho\in\{\mathrm{tr},1,\ldots,d\}\) & Trace readout or word-local mode index. \\
\(d_{\bullet,\mathrm{tr}}^{(w)}(t,s),d_{\bullet,r}^{(w)}(t,s)\) & Trace distance and distance along word-local mode \(r\). \\
\(\mathcal M_{\mathrm{mean}}^{(w)}\) & Word-local path operator \(\sum_{t=1}^{T-1}\mathbf M_{\mathrm{mean}}^{(w)}(t,t+1)\). \\
\((\lambda_r^{(w)},\bm u_r^{(w)})\) & Eigenvalue and eigenvector of \(\mathcal M_{\mathrm{mean}}^{(w)}\). The eigenvector defines word-local mode \(r\). \\
\midrule
\multicolumn{2}{@{}l}{\textit{Gaussian specialization and recovery}} \\
\(\mu_{t,k}^{(w)},\Sigma_{t,k}^{(w)}\) & Mean and covariance of Gaussian sense component \(\nu_{t,k}^{(w)}\). \\
\(A_{t,k,\ell}^{(w)},T_{t,k,\ell}^{(w)}\) & Linear part and affine map of the adjacent optimal Gaussian coupling. \\
\(M_{t,s}^{(w)}\) & Product of adjacent Gaussian linear maps along a latent sense path from \(t\) to \(s\), transposed according to Appendix~\ref{app:gaussian_expressions}. \\
\(\Xi^{(w)}(t,s;k,\ell)\) & Endpoint conditioned cross-period covariance used in \(\mathbf C_{\mathrm{dev}}^{(w)}(t,s;k,\ell)\). \\
\(n_t^{(w)},n^{(w)}\) & Period sample size and minimum sample size \(\min_t n_t^{(w)}\) used in the recovery theorem. \\
\(\widehat{\theta}\) & Estimated version of a population quantity \(\theta\). \\
\midrule
\multicolumn{2}{@{}l}{\textit{Court readouts}} \\
\(d_t,z_t\) & Adjacent full-trace distance and its within-lemma standardized value. \\
\(a_{k\ell},\alpha_{k\ell}\) & Squared full-trace contribution of pair \((k,\ell)\) at the selected peak and its share of total peak variation. \\
\(E_{\max}^{(w)}\) & Maximum transition enrichment \(\max_{q_{k\ell}>0}\alpha_{k\ell}/q_{k\ell}\). \\
\(e_r,\tau_r\) & Share of mean-path energy carried by mode \(r\) and the adjacent period at which that mode peaks. \\
\(S_2\) & Secondary-mode temporal-separation score used when the two leading retained modes peak in different periods. \\
\bottomrule
\end{tabular}
\end{table*}

\subsection{Structural comparison with related approaches}

Table~\ref{tab:method_comparison_compact} compares the structures provided by closely related approaches. A checkmark denotes a directly provided capability, $\triangle$ a related but non-equivalent object, and a dash a capability that is not directly provided. The comparison concerns the joint set of capabilities rather than any single column. Compositional coupling means that adjacent and non-adjacent usage--component correspondences arise from one multi-period process rather than independent fits. An exact mechanism split and transition attribution decompose the same change quantity by mechanism and transported component pair, respectively. Temporal modes are word-local directions with period local activity profiles. For SynFlow, $\triangle$ denotes related filler cluster continuity and value contributions rather than transported usage--component accounting.

\begin{table*}[t]
\centering
\caption{Structural comparison with closely related approaches.}
\label{tab:method_comparison_compact}
\small
\setlength{\tabcolsep}{2.5pt}
\renewcommand{\arraystretch}{1.15}
\begin{tabularx}{\textwidth}{@{}>{\raggedright\arraybackslash}X*{6}{c}@{}}
\toprule
& \multicolumn{3}{c}{\textbf{Temporal representation}}
& \multicolumn{3}{c}{\textbf{Attribution from that representation}} \\
\cmidrule(lr){2-4}
\cmidrule(lr){5-7}
\textbf{Approach}
& \shortstack{\textbf{Multi-}\\\textbf{period}}
& \shortstack{\textbf{Usage/component}\\\textbf{correspondence}}
& \shortstack{\textbf{Compositional}\\\textbf{coupling}}
& \shortstack{\textbf{Exact mechanism}\\\textbf{split}}
& \shortstack{\textbf{Transition}\\\textbf{attribution}}
& \shortstack{\textbf{Temporal}\\\textbf{modes}} \\
\midrule
Dynamic and evolving sense models
\citep{frermann2016bayesian,hu2019diachronic}
& \checkmark & \checkmark & -- & $\triangle$ & -- & -- \\

WiDiD \citep{periti2022done,periti2025studying}
& \checkmark & \checkmark & -- & -- & $\triangle$ & -- \\

Multi-period similarity analysis \citep{kiyama2025analyzing}
& \checkmark & -- & -- & -- & -- & -- \\

SynFlow \citep{phantat2026synflow}
& \checkmark & $\triangle$ & -- & $\triangle$ & $\triangle$ & -- \\

Pairwise OT \citep{montariol2021transport}
& $\triangle$ & \checkmark & -- & -- & -- & -- \\

ConShift \citep{arrington2025conshift}
& $\triangle$ & \checkmark & -- & $\triangle$ & $\triangle$ & -- \\

UOT / SUS \citep{kishino2025quantifying}
& $\triangle$ & \checkmark & -- & -- & $\triangle$ & -- \\

Directional methods
\citep{aida2025investigating,aida2025scdtour}
& -- & -- & -- & -- & -- & $\triangle$ \\

MENT \citep{ezoe2026multiscale}
& \checkmark & -- & -- & $\triangle$ & -- & \checkmark \\

\midrule
\textbf{CUSP (ours)}
& \checkmark & \checkmark & \checkmark
& \checkmark & \checkmark & \checkmark \\
\bottomrule
\end{tabularx}
\end{table*}

\section{Identifiability of the sense mixture representation}
\label{subsec:sense_mixture_identifiability}

The mixture representation of a period specific word distribution is not unique in general, since different parameterizations may induce the same distribution. The following assumption formalizes the identifiability condition stated in the main text: the mixture representation is uniquely determined up to permutation
of the sense labels.

\begin{assumption}[Identifiability up to sense-label permutation]
\label{ass:identifiability}

Suppose that two mixture representations
\[
\sum_{k=1}^{K_t^{(w)}}
\pi_{t,k}^{(w)}
\nu_{t,k}^{(w)}
\qquad\text{and}\qquad
\sum_{k=1}^{\tilde K_t^{(w)}}
\tilde{\pi}_{t,k}^{(w)}
\tilde{\nu}_{t,k}^{(w)}
\]
represent the same period specific word distribution. Then
\[
K_t^{(w)}=\tilde K_t^{(w)},
\]
and there exists a permutation
\(\sigma\) of
\(\{1,\ldots,K_t^{(w)}\}\)
such that
\[
\tilde{\pi}_{t,k}^{(w)}
=
\pi_{t,\sigma(k)}^{(w)},
\qquad
\tilde{\nu}_{t,k}^{(w)}
=
\nu_{t,\sigma(k)}^{(w)},
\qquad
k=1,\ldots,K_t^{(w)}.
\]
\end{assumption}

Consequently, all label-dependent quantities are defined up to the corresponding label permutation. In the statistical recovery analysis, whenever population and estimated quantities are compared, we implicitly choose the permutation that aligns the corresponding sense labels.

The Gaussian-mixture specialization used in the statistical recovery analysis provides a concrete admissible class of mixture distributions and their parameterizations satisfying the identifiability condition above. The restrictions defining this class correspond to the first part of Assumption~\ref{ass:gaussian_estimation_regularity}.

\begin{definition}[Admissible Gaussian mixture class]
\label{def:gmm_class}

The admissible Gaussian mixture class consists of Gaussian mixture
distributions
\[
\gamma
=
\sum_{k=1}^{K}
\pi_k\nu_k,
\qquad
\nu_k=\mathcal N(m_k,\Sigma_k),
\]
together with parameterizations satisfying
\begin{itemize}
\item
\(
\pi_k>0,
\qquad
\sum_{k=1}^{K}\pi_k=1;
\)
\item
\(
\Sigma_k\succ 0,
\qquad k=1,\ldots,K;
\)
\item the component parameters are pairwise distinct:
\[
(m_k,\Sigma_k)\neq(m_\ell,\Sigma_\ell),
\qquad
k\neq\ell.
\]
\end{itemize}
\end{definition}

\section{Proofs of basic properties of the coupled usage-sense processes}
\label{app:basic_process_properties}

We prove Proposition~\ref{prop:marginal_preservation}, which establishes preservation of the prescribed sense prevalences and contextual component distributions at each period, and consequently of the observed usage distributions as the period-wise marginals of the resulting joint process.

\begin{proof}
Fix \(w\in V\), and suppress the superscript \((w)\) throughout the proof.
We first prove the preservation of the sense marginals by induction.
By construction,
\[
\Pr(Z_1=k)=\pi_{1,k}
\]
for every \(k\in[K_1]\), so the claim holds at \(t=1\).
If \(\Pr(Z_t=k)=\pi_{t,k}\) for every \(k\in[K_t]\), then
\[
\begin{aligned}
\Pr(Z_{t+1}=\ell)
&=
\sum_{k\in[K_t]} \Pr(Z_t=k)P_t(k,\ell)\\
&=
\sum_{k\in[K_t]} \pi_{t,k}
\frac{Q_t(k,\ell)}{\pi_{t,k}}
=
\sum_{k\in[K_t]} Q_t(k,\ell)
=
\pi_{t+1,\ell},
\end{aligned}
\]
where \(P_t(k,\ell)=Q_t(k,\ell)/\pi_{t,k}\) is the transition probability, and the last equality follows from the second marginal constraint in
\ref{eq:sense_transition_distribution}, namely
\(\sum_{k\in[K_t]} Q_t(k,\ell)=\pi_{t+1,\ell}\).
Thus, by induction,
\[
\Pr(Z_t=k)=\pi_{t,k}
\]
for all \(t\in[T]\) and \(k\in[K_t]\).

We next prove the preservation of the conditional contextual
distributions by induction. Fix a sense path \(z_{1:T}\) with positive probability.
By construction,
\[
\mathcal L(X_1\mid Z_{1:T}=z_{1:T})=\nu_{1,z_1},
\]
so the claim holds at \(t=1\). Suppose that
\[
\mathcal L(X_t\mid Z_{1:T}=z_{1:T})=\nu_{t,z_t}.
\]
Then, for any measurable set \(A\subseteq\mathbb R^d\),
\[
\begin{aligned}
\Pr(X_{t+1}\in A\mid Z_{1:T}=z_{1:T})
&=
\int
\mathcal K_{t,z_t,z_{t+1}}(x,A)\,
\nu_{t,z_t}(dx)\\
&=\int
\mathbf{1}_A(y)\,
\eta_{t,z_t,z_{t+1}}(dx,dy)\\
&=
\nu_{t+1,z_{t+1}}(A),
\end{aligned}
\]
where we used the transition kernel defined by
\[
\eta_{t,z_t,z_{t+1}}(dx,dy)
=
\nu_{t,z_t}(dx)
\mathcal K_{t,z_t,z_{t+1}}(x,dy),
\]
and the last equality follows from the marginal constraint in \ref{eq:pairwise_distribution_coupling}, namely \(\eta_{t,z_t,z_{t+1}}\in \Pi(\nu_{t,z_t},\nu_{t+1,z_{t+1}})\).
Thus, by induction,
\[
\mathcal L(X_t\mid Z_{1:T}=z_{1:T})=\nu_{t,z_t}
\]
for all \(t\in[T]\), establishing the desired preservation of the conditional contextual distributions at each period.

Finally, using the sense marginal preservation and the conditional
distribution established above, the marginal law of \(X_t\) is
\[
\begin{aligned}
\mathcal L(X_t)
&=
\sum_{z_{1:T}}
\Pr(Z_{1:T}=z_{1:T})\,
\mathcal L(X_t\mid Z_{1:T}=z_{1:T})\\
&=
\sum_k
\left(
\sum_{z_{1:T}:z_t=k}
\Pr(Z_{1:T}=z_{1:T})
\right)
\nu_{t,k}\\
&=
\sum_{k\in[K_t]}
\pi_{t,k}\nu_{t,k}
=
\gamma_t.
\end{aligned}
\]
Consequently,
\[
\mathcal L(X_1,\ldots,X_T)
\in
\Pi(\gamma_1,\ldots,\gamma_T),
\]
which proves the proposition.
\end{proof}

We prove the reconstruction and conditional centering properties stated in
\eqref{eq:full_mean_dev_decomposition}. The reconstruction identity
follows directly from the definitions, while the conditional centering property follows from the preservation
of the conditional marginals established in
Proposition~\ref{prop:marginal_preservation}.

\begin{proof}
By definition,
\[
\phi_{\mathrm{full}}^{(w)}(t)
=
\phi_{\mathrm{mean}}^{(w)}(t)
+
\phi_{\mathrm{dev}}^{(w)}(t),
\]
so the reconstruction identity holds.

Moreover, by Proposition~\ref{prop:marginal_preservation},
for any sense path \(z_{1:T}\) with positive probability,
\[
\begin{aligned}
\mathbb E\left[
X_t^{(w)}
\mid
Z_{1:T}^{(w)}=z_{1:T}
\right]
=
\int x\,d\nu_{t,z_t}^{(w)}(x)
=
m_{t,z_t}^{(w)}.
\end{aligned}
\]
Therefore,
\[
\begin{aligned}
\mathbb E\left[
\phi_{\mathrm{dev}}^{(w)}(t)
\mid
Z_{1:T}^{(w)}
\right]
&=
\mathbb E\left[
X_t^{(w)}
-
m_{t,Z_t^{(w)}}^{(w)}
\mid
Z_{1:T}^{(w)}
\right]\\
&=
m_{t,Z_t^{(w)}}^{(w)}
-
m_{t,Z_t^{(w)}}^{(w)}
=
0.
\end{aligned}
\]
Thus, both identities in
\eqref{eq:full_mean_dev_decomposition} hold.
\end{proof}

\section{Proofs of exact mechanistic and sense-transition attribution}
\label{app:decomposition_properties}

We first prove Theorem~\ref{thm:exact_mean_deviation_decomposition}. The key
step is that the deviation process is conditionally centered given the
full latent sense sequence, so the cross terms between the mean and
deviation processes vanish.

\begin{proof}
Fix \(w\in V\) and \(t,s\in[T]\). By the reconstruction identity in
\eqref{eq:full_mean_dev_decomposition},
\[
\begin{aligned}
\Delta\phi_{\mathrm{full}}^{(w)}(t,s)
&=
\Delta\phi_{\mathrm{mean}}^{(w)}(t,s)
+
\Delta\phi_{\mathrm{dev}}^{(w)}(t,s).
\end{aligned}
\]
Therefore, expanding the second-moment operator gives
\[
\begin{aligned}
\mathbf M_{\mathrm{full}}^{(w)}(t,s)
&=
\mathbf M_{\mathrm{mean}}^{(w)}(t,s)
+
\mathbf M_{\mathrm{dev}}^{(w)}(t,s)
\\
&\quad+
\mathbb E\left[
\Delta\phi_{\mathrm{mean}}^{(w)}(t,s)
\Delta\phi_{\mathrm{dev}}^{(w)}(t,s)^\top
\right]
\\
&\quad+
\mathbb E\left[
\Delta\phi_{\mathrm{dev}}^{(w)}(t,s)
\Delta\phi_{\mathrm{mean}}^{(w)}(t,s)^\top
\right].
\end{aligned}
\]
We first consider the first cross term.
\[
\begin{aligned}
&
\mathbb E\left[
\Delta\phi_{\mathrm{mean}}^{(w)}(t,s)
\Delta\phi_{\mathrm{dev}}^{(w)}(t,s)^\top
\right]
\\
&=
\mathbb E\left[
\left(
m_{t,Z_t^{(w)}}^{(w)}
-
m_{s,Z_s^{(w)}}^{(w)}
\right)
\mathbb E\left[
\Delta\phi_{\mathrm{dev}}^{(w)}(t,s)^\top
\mid
Z_{1:T}^{(w)}
\right]
\right]
\\
&=0,
\end{aligned}
\]
where the last equality follows from \eqref{eq:full_mean_dev_decomposition}.
The second cross term vanishes
analogously. Consequently,
\[
\mathbf M_{\mathrm{full}}^{(w)}(t,s)
=
\mathbf M_{\mathrm{mean}}^{(w)}(t,s)
+
\mathbf M_{\mathrm{dev}}^{(w)}(t,s),
\]
which proves \eqref{eq:mean_dev_operator_decomposition}.

Applying the trace and the quadratic form
\(\bm u_r^\top(\cdot)\bm u_r\) gives
\eqref{eq:mean_dev_distance_decomposition} for every
$\rho\in\{\mathrm{tr},1,\ldots,d\}$.
\end{proof}

We next prove Theorem~\ref{thm:exact_sense_transition_attribution}. We first
decompose each second-moment operator by conditioning on the latent sense
path and then group the resulting terms according to the pair
$(Z_t^{(w)},Z_s^{(w)})$.

\begin{proof}
Fix $w\in V$, $t,s\in[T]$, and
$\bullet\in\{\mathrm{mean},\mathrm{dev}\}$.

For the mean process,
\[
\begin{aligned}
\mathbf M_{\mathrm{mean}}^{(w)}(t,s)
&=
\sum_{z_{1:T}}
\Pr(Z_{1:T}^{(w)}=z_{1:T})
\left(
m_{t,z_t}^{(w)}-m_{s,z_s}^{(w)}
\right)
\left(
m_{t,z_t}^{(w)}-m_{s,z_s}^{(w)}
\right)^\top
\\
&=
\sum_{k=1}^{K_t^{(w)}}
\sum_{\ell=1}^{K_s^{(w)}}
\left(
\sum_{z_{1:T}:z_t=k,\;z_s=\ell}
\Pr(Z_{1:T}^{(w)}=z_{1:T})
\right)
\\
&\qquad\qquad\cdot
\left(
m_{t,k}^{(w)}-m_{s,\ell}^{(w)}
\right)
\left(
m_{t,k}^{(w)}-m_{s,\ell}^{(w)}
\right)^\top
\\
&=
\sum_{k=1}^{K_t^{(w)}}
\sum_{\ell=1}^{K_s^{(w)}}
q^{(w)}(t,s;k,\ell)
\mathbf C_{\mathrm{mean}}^{(w)}(t,s;k,\ell).
\end{aligned}
\]

For the deviation process, first note that, for any
$k\in[K_t^{(w)}]$ and $\ell\in[K_s^{(w)}]$,
\[
\begin{aligned}
&\mathbb E\left[
X_t^{(w)}-X_s^{(w)}
\mid
Z_t^{(w)}=k,Z_s^{(w)}=\ell
\right]
\\
&=
\sum_{z_{1:T}:z_t=k,\;z_s=\ell}
\Pr\left(
Z_{1:T}^{(w)}=z_{1:T}
\mid
Z_t^{(w)}=k,Z_s^{(w)}=\ell
\right)
\\
&\qquad\qquad\cdot
\mathbb E\left[
X_t^{(w)}-X_s^{(w)}
\mid
Z_{1:T}^{(w)}=z_{1:T}
\right]
\\
&=
\sum_{z_{1:T}:z_t=k,\;z_s=\ell}
\Pr\left(
Z_{1:T}^{(w)}=z_{1:T}
\mid
Z_t^{(w)}=k,Z_s^{(w)}=\ell
\right)
\\
&\qquad\qquad\cdot
\left(
m_{t,z_t}^{(w)}-m_{s,z_s}^{(w)}
\right)
\\
&=
m_{t,k}^{(w)}-m_{s,\ell}^{(w)},
\end{aligned}
\]
where the second equality follows from the conditional marginal
preservation in Proposition~\ref{prop:marginal_preservation}. Hence,
\[
\begin{aligned}
&
\mathbb E\left[
\Delta\phi_{\mathrm{dev}}^{(w)}(t,s)
\Delta\phi_{\mathrm{dev}}^{(w)}(t,s)^\top
\mid
Z_t^{(w)}=k,Z_s^{(w)}=\ell
\right]
\\
&=
\mathbb E\left[
\left\{
X_t^{(w)}-X_s^{(w)}
-
\left(
m_{t,k}^{(w)}-m_{s,\ell}^{(w)}
\right)
\right\}
\left\{
X_t^{(w)}-X_s^{(w)}
-
\left(
m_{t,k}^{(w)}-m_{s,\ell}^{(w)}
\right)
\right\}^{\top}
\middle|
Z_t^{(w)}=k,Z_s^{(w)}=\ell
\right]
\\
&=
\operatorname{Cov}\left(
X_t^{(w)}-X_s^{(w)}
\mid
Z_t^{(w)}=k,Z_s^{(w)}=\ell
\right)
\\
&=
\mathbf C_{\mathrm{dev}}^{(w)}(t,s;k,\ell).
\end{aligned}
\]
Therefore, conditioning on the pair $(Z_t^{(w)},Z_s^{(w)})$ gives
\[
\begin{aligned}
\mathbf M_{\mathrm{dev}}^{(w)}(t,s)
&=
\sum_{k=1}^{K_t^{(w)}}
\sum_{\ell=1}^{K_s^{(w)}}
q^{(w)}(t,s;k,\ell)
\\
&\qquad\qquad\cdot
\mathbb E\left[
\Delta\phi_{\mathrm{dev}}^{(w)}(t,s)
\Delta\phi_{\mathrm{dev}}^{(w)}(t,s)^\top
\mid
Z_t^{(w)}=k,Z_s^{(w)}=\ell
\right]
\\
&=
\sum_{k=1}^{K_t^{(w)}}
\sum_{\ell=1}^{K_s^{(w)}}
q^{(w)}(t,s;k,\ell)
\mathbf C_{\mathrm{dev}}^{(w)}(t,s;k,\ell).
\end{aligned}
\]

Thus, \eqref{eq:sense_transition_operator_decomposition} follows for both
$\bullet\in\{\mathrm{mean},\mathrm{dev}\}$.

Applying the trace and the quadratic form
$\bm u_r^{(w)\top}(\cdot)\bm u_r^{(w)}$ gives
\eqref{eq:trace_sense_transition}
and~\eqref{eq:mode_sense_transition}.
\end{proof}

\section{Proofs of Properties of Trace and Mode-Wise Geometry}
\label{appx:geometric_characterizations}

\subsection{Validity and properties of the induced semantic distances}
\label{subsec:semantic_distance_properties}

The semantic distances introduced in the main text define the geometry of
semantic change on the period index set. This subsection establishes the pseudometric properties of the semantic distances induced by the corresponding \(L^2\) random-vector representations.

%This subsection establishes two basic properties of these distances. First, since they are induced by \(L^2\) distances between the corresponding random-vector representations, they satisfy the pseudo-metric properties. Second, they are invariant under global changes of coordinates in the embedding space, provided that the coupled process construction is equivariant under the corresponding transformations.

The first property concerns the pseudo-metric structure induced by the
random-vector representations.
\begin{proposition}[Pseudo-metric properties of trace and mode-wise semantic distances]
\label{prop:semantic_distance_pseudometric}

For every \(w\in V\) and
\(\bullet\in\{\mathrm{full},\mathrm{mean},\mathrm{dev}\}\),
the quantities
\[
d_{\bullet,\mathrm{tr}}^{(w)}(t,s),
\qquad
d_{\bullet,r}^{(w)}(t,s),\quad r\in[d],
\]
define finite pseudo-metrics on \([T]\).

Moreover, the zero-distance conditions are given by
\[
d_{\bullet,\mathrm{tr}}^{(w)}(t,s)=0
\iff
\phi_{\bullet}^{(w)}(t)
=
\phi_{\bullet}^{(w)}(s)
\quad\text{almost surely},
\]
and
\[
d_{\bullet,r}^{(w)}(t,s)=0
\iff
\bm u_r^\top
\phi_{\bullet}^{(w)}(t)
=
\bm u_r^\top
\phi_{\bullet}^{(w)}(s)
\quad\text{almost surely}.
\]
\end{proposition}

\textbf{Remark.}
Zero distance implies equality only in the corresponding random-vector representation and does not necessarily imply \(t=s\). For the mode-wise distances, the zero-distance condition only concerns the projection onto the corresponding direction \(\bm u_r\).

\begin{proof}
Since
\(\nu_{t,k}^{(w)}\in\mathcal P_2(\mathbb R^d)\),
the mixture distribution satisfies
\(\gamma_t^{(w)}\in\mathcal P_2(\mathbb R^d)\).
By the marginal consistency of the full process in \eqref{eq:joint_coupling_marginals},
\(\phi_{\mathrm{full}}^{(w)}(t)\in L^2(\Omega;\mathbb R^d)\).
Moreover,
\(\phi_{\mathrm{mean}}^{(w)}(t)\in L^2(\Omega;\mathbb R^d)\),
since it takes values in the finite set of sense centers
\(\{m_{t,k}^{(w)}\}_{k=1}^{K_t^{(w)}}\).
Therefore,
\[
\phi_{\mathrm{dev}}^{(w)}(t)
=
\phi_{\mathrm{full}}^{(w)}(t)
-
\phi_{\mathrm{mean}}^{(w)}(t)
\in L^2(\Omega;\mathbb R^d) .
\]
Hence, all the following \(L^2\) distances are well-defined.

For the trace distance, the definition gives
\[
d_{\bullet,\mathrm{tr}}^{(w)}(t,s)
=
\left(
\mathbb E
\left[
\left\|
\phi_{\bullet}^{(w)}(t)
-
\phi_{\bullet}^{(w)}(s)
\right\|^2
\right]
\right)^{1/2}
=
\left\|
\phi_{\bullet}^{(w)}(t)
-
\phi_{\bullet}^{(w)}(s)
\right\|_{L^2(\Omega;\mathbb R^d)}.
\]
Hence,
\(d_{\bullet,\mathrm{tr}}^{(w)}\) is the \(L^2\) distance between the
random vectors
\(\phi_{\bullet}^{(w)}(t)\) and
\(\phi_{\bullet}^{(w)}(s)\).
Nonnegativity and symmetry follow directly from the properties of the norm,
and the triangle inequality follows from Minkowski's inequality.

Similarly,
\[
d_{\bullet,r}^{(w)}(t,s)
=
\left\|
\bm u_r^\top
\phi_{\bullet}^{(w)}(t)
-
\bm u_r^\top
\phi_{\bullet}^{(w)}(s)
\right\|_{L^2(\Omega;\mathbb R)},
\]
which is the \(L^2\) distance between the scalar random variables obtained by
projecting the process onto the direction \(\bm u_r\). Therefore, the same
arguments establish nonnegativity, symmetry, and the triangle inequality for
\(d_{\bullet,r}^{(w)}\).

Finally, an \(L^2\) norm is zero if and only if the corresponding random elements agree almost surely. Applying this property to the vector-valued and projected representations gives the stated zero-distance conditions.
\end{proof}

\subsection{Proofs of the decomposition and  characterization of the semantic geometry}
\label{subsec:semantic_geometry_proofs}

This section provides proofs of the characterization and decomposition
results for the semantic geometry introduced in the main text. The
proofs follow the same arguments as those in \cite{ezoe2026multiscale}.

The following proof establishes the decomposition of the total semantic
change into orthogonal mode-wise contributions.

\begin{proof}[Proof of \eqref{eq:trace_mode_decomposition}]
Since \(\{\bm u_r\}_{r=1}^d\) is an orthonormal basis,
\[
\operatorname{tr}\mathbf M_{\bullet}^{(w)}(t,s)
=
\sum_{r=1}^d
(\bm u_r^{(w)})^\top
\mathbf M_{\bullet}^{(w)}(t,s)
\bm u_r^{(w)}.
\]
Using the definitions of
\(d_{\bullet,\mathrm{tr}}^{(w)}(t,s)\) and
\(d_{\bullet,r}^{(w)}(t,s)\) yields the claim.
\end{proof}

%The following proof establishes the spectral characterization of the shared semantic modes induced by the aggregate second-moment operator.
The following proof establishes the spectral characterization of word-local modes induced by each word's aggregated mean operator.

\begin{proof}[Proof of Proposition~\ref{prop:word_local_modes}]
By definition,
\[
\mathcal M_{\mathrm{mean}}^{(w)}
=
\sum_{t=1}^{T-1}
\mathbf M_{\mathrm{mean}}^{(w)}(t,t+1).
\]
Since
\[
\mathcal M_{\mathrm{mean}}^{(w)}\bm u_r^{(w)}
=
\lambda_r^{(w)}\bm u_r^{(w)}
\quad\text{and}\quad
\|\bm u_r^{(w)}\|=1,
\]
we have
\[
\begin{aligned}
\lambda_r^{(w)}
&=
(\bm u_r^{(w)})^\top
\mathcal M_{\mathrm{mean}}^{(w)}
\bm u_r^{(w)}\\
&=
\sum_{t=1}^{T-1}
(\bm u_r^{(w)})^\top
\mathbf M_{\mathrm{mean}}^{(w)}(t,t+1)
\bm u_r^{(w)}\\
&=
\sum_{t=1}^{T-1}
\left(
d_{\mathrm{mean},r}^{(w)}(t,t+1)
\right)^2,
\end{aligned}
\]
which proves \eqref{eq:eigenvalue_accounting}.

Moreover, since
\(\mathcal M_{\mathrm{mean}}^{(w)}\) is symmetric,
the Courant--Fischer variational characterization gives
\[
\bm u_r^{(w)}
\in
\arg\max_{\substack{\|\bm u\|=1\\
\bm u\perp\bm u_1^{(w)},\ldots,\bm u_{r-1}^{(w)}}}
\bm u^\top
\mathcal M_{\mathrm{mean}}^{(w)}
\bm u.
\]
Substituting the definition of
\(\mathcal M_{\mathrm{mean}}^{(w)}\) yields
\[
\bm u_r^{(w)}
\in
\arg\max_{\substack{\|\bm u\|=1\\
\bm u\perp\bm u_1^{(w)},\ldots,\bm u_{r-1}^{(w)}}}
\sum_{t=1}^{T-1}
\bm u^\top
\mathbf M_{\mathrm{mean}}^{(w)}(t,t+1)
\bm u,
\]
which proves \eqref{eq:word_local_variational_characterization}.
\end{proof}

\section{Explicit expressions under the Gaussian specialization}
\label{app:gaussian_expressions}

Under the Gaussian and quadratic-cost specialization in
Section~\ref{sec:change_operators}, both the contextual coupling cost and
the sense-pair second-moment components admit explicit expressions.

First, the contextual coupling problem in
\eqref{eq:pairwise_distribution_coupling} reduces to the quadratic
\(2\)-Wasserstein problem between Gaussian measures. As established, e.g.,
in \cite{okano2024distribution,santambrogio2017euclidean}, the optimal
coupling exists and is unique under positive-definite covariance matrices,
and is induced by the affine optimal transport map given in
Proposition~\ref{prop:gaussian_optimal_coupling}. Its minimum cost is
therefore
\begin{align}
C_{t,k,\ell}^{(w)}
&=
W_2^2\!\left(
\nu_{t,k}^{(w)},\nu_{t+1,\ell}^{(w)}
\right)
\nonumber\\
&=
\left\|
\mu_{t,k}^{(w)}-\mu_{t+1,\ell}^{(w)}
\right\|^2
+
\operatorname{Tr}\!\left(
\Sigma_{t,k}^{(w)}
+
\Sigma_{t+1,\ell}^{(w)}
-
2\left[
(\Sigma_{t,k}^{(w)})^{1/2}
\Sigma_{t+1,\ell}^{(w)}
(\Sigma_{t,k}^{(w)})^{1/2}
\right]^{1/2}
\right).
%\Sigma_{t,k}^{(w)}
%+
%\Sigma_{t+1,\ell}^{(w)}
%-
%2
%\left[
%(\Sigma_{t,k}^{(w)})^{1/2}
%\Sigma_{s,\ell}^{(w)}
%(\Sigma_{t,k}^{(w)})^{1/2}
%\right]^{1/2}
%\right).
\label{eq:gaussian_sense_transition_cost}
\end{align}

We next give explicit expressions for the sense-pair second-moment
components used in
Theorem~\ref{thm:exact_sense_transition_attribution}.

\begin{proposition}[Explicit expressions for sense-pair second moments]
\label{prop:explicit_sense_components}
%Under the Gaussian specialization, for every \(w\in V\), \(t<s\), and
%\(k\in[K_t^{(w)}]\), \(\ell\in[K_s^{(w)}]\), the mean component is
Under the Gaussian specialization, for every \(w\in V\), \(t<s\), and
pair \((k,\ell)\) with \(q^{(w)}(t,s;k,\ell)>0\), the mean component is
\begin{equation}
\mathbf C_{\mathrm{mean}}^{(w)}(t,s;k,\ell)
=
\left(
\mu_{t,k}^{(w)}-\mu_{s,\ell}^{(w)}
\right)
\left(
\mu_{t,k}^{(w)}-\mu_{s,\ell}^{(w)}
\right)^\top,
\label{eq:explicit_C_mean}
\end{equation}
and the deviation component is
\begin{equation}
\mathbf C_{\mathrm{dev}}^{(w)}(t,s;k,\ell)
=
\Sigma_{t,k}^{(w)}
+
\Sigma_{s,\ell}^{(w)}
-
\Xi^{(w)}(t,s;k,\ell)
-
\Xi^{(w)}(t,s;k,\ell)^\top,
\label{eq:explicit_C_dev}
\end{equation}
where
\begin{align}
M_{t,s}^{(w)}
&:=
\left(
A_{s-1,Z_{s-1}^{(w)},Z_s^{(w)}}^{(w)}
\cdots
A_{t,Z_t^{(w)},Z_{t+1}^{(w)}}^{(w)}
\right)^\top,
\label{eq:app_def_M_ts}
\\
\Xi^{(w)}(t,s;k,\ell)
&:=
\Sigma_{t,k}^{(w)}
\mathbb E\!\left[
M_{t,s}^{(w)}
\middle|
Z_t^{(w)}=k,Z_s^{(w)}=\ell
\right]
\label{eq:app_def_Xi}
\end{align}
\end{proposition}

\begin{proof}
The expression for
\(\mathbf C_{\mathrm{mean}}^{(w)}(t,s;k,\ell)\) follows directly from the means of the Gaussian sense components. 

For the deviation component, apply the law of total covariance by
conditioning further on the complete latent sense sequence:
\begin{align}
\mathbf C_{\mathrm{dev}}^{(w)}(t,s;k,\ell)
&=
\mathbb E\!\left[
\operatorname{Cov}\!\left(
X_t^{(w)}-X_s^{(w)}
\mid
Z_{1:T}^{(w)}
\right)
\middle|
Z_t^{(w)}=k,Z_s^{(w)}=\ell
\right]
\nonumber\\
&\quad+
\operatorname{Cov}\!\left(
\mathbb E\!\left[
X_t^{(w)}-X_s^{(w)}
\mid
Z_{1:T}^{(w)}
\right]
\middle|
Z_t^{(w)}=k,Z_s^{(w)}=\ell
\right).
\label{eq:app_dev_cov_total_covariance}
\end{align}
By the conditional marginals established in Proposition~\ref{prop:marginal_preservation},
\begin{align}
\operatorname{Cov}\!\left(
\mathbb E\!\left[
X_t^{(w)}-X_s^{(w)}
\mid
Z_{1:T}^{(w)}
\right]
\middle|
Z_t^{(w)}=k,Z_s^{(w)}=\ell
\right)\\
=
\operatorname{Cov}\!\left(
\mu_{t,k}^{(w)}-\mu_{s,\ell}^{(w)}
\middle|
Z_t^{(w)}=k,Z_s^{(w)}=\ell
\right)
=
0,
\end{align}
thus the second term vanishes.

Now fix a latent sense sequence \(z_{1:T}\) satisfying
\(z_t=k\) and \(z_s=\ell\). The conditional covariance of the displacement
decomposes as
\begin{align}
&
\operatorname{Cov}\!\left(
X_t^{(w)}-X_s^{(w)}
\mid
Z_{1:T}^{(w)}=z_{1:T}
\right)
\nonumber\\
&=
\operatorname{Var}\!\left(
X_t^{(w)}
\mid
Z_{1:T}^{(w)}=z_{1:T}
\right)
+
\operatorname{Var}\!\left(
X_s^{(w)}
\mid
Z_{1:T}^{(w)}=z_{1:T}
\right)
\nonumber\\
&\quad
-
\operatorname{Cov}\!\left(
X_t^{(w)},X_s^{(w)}
\mid
Z_{1:T}^{(w)}=z_{1:T}
\right)
-
\operatorname{Cov}\!\left(
X_s^{(w)},X_t^{(w)}
\mid
Z_{1:T}^{(w)}=z_{1:T}
\right).
\label{eq:app_conditional_dev_decomposition}
\end{align}
By Proposition~\ref{prop:marginal_preservation},
\begin{align}
\operatorname{Var}\!\left(
X_t^{(w)}
\mid
Z_{1:T}^{(w)}=z_{1:T}
\right)
=
\Sigma_{t,k}^{(w)},
\qquad
\operatorname{Var}\!\left(
X_s^{(w)}
\mid
Z_{1:T}^{(w)}=z_{1:T}
\right)
&=
\Sigma_{s,\ell}^{(w)}.
\end{align}

Moreover, Proposition~\ref{prop:gaussian_optimal_coupling} gives the
affine optimal transport map for each adjacent sense transition.
Consequently, the Markov composition along the fixed sense path gives
\begin{align}
X_s^{(w)}
&=
\mu_{s,\ell}^{(w)}
+
A_{s-1,z_{s-1},z_s}^{(w)}
\cdots
A_{t,z_t,z_{t+1}}^{(w)}
\left(
X_t^{(w)}-\mu_{t,k}^{(w)}
\right).
\label{eq:app_composed_gaussian_map}
\end{align}
Taking the conditional covariance gives
\begin{align}
&
\operatorname{Cov}\!\left(
X_t^{(w)},X_s^{(w)}
\mid
Z_{1:T}^{(w)}=z_{1:T}
\right) = \operatorname{Cov}\!\left(
X_s^{(w)},X_t^{(w)}
\mid
Z_{1:T}^{(w)}=z_{1:T}
\right)^\top
\nonumber\\
&\qquad=
\Sigma_{t,k}^{(w)}
\left(
A_{s-1,z_{s-1},z_s}^{(w)}
\cdots
A_{t,z_t,z_{t+1}}^{(w)}
\right)^\top,
\label{eq:app_cross_cov_ts}
\end{align}

Substituting these expressions into
\eqref{eq:app_conditional_dev_decomposition}, we obtain
\begin{align}
&
\operatorname{Cov}\!\left(
X_t^{(w)}-X_s^{(w)}
\mid
Z_{1:T}^{(w)}=z_{1:T}
\right)
\nonumber\\
&=
\Sigma_{t,k}^{(w)}
+
\Sigma_{s,\ell}^{(w)}
-
\Sigma_{t,k}^{(w)}
\left(
A_{s-1,z_{s-1},z_s}^{(w)}
\cdots
A_{t,z_t,z_{t+1}}^{(w)}
\right)^\top
\nonumber\\
&\quad
-
\left[
\Sigma_{t,k}^{(w)}
\left(
A_{s-1,z_{s-1},z_s}^{(w)}
\cdots
A_{t,z_t,z_{t+1}}^{(w)}
\right)^\top
\right]^\top.
\label{eq:app_conditional_dev_cov_path}
\end{align}

Combining the above expressions with
\eqref{eq:app_dev_cov_total_covariance}, we obtain
\begin{align}
\mathbf C_{\mathrm{dev}}^{(w)}(t,s;k,\ell)
&=
\Sigma_{t,k}^{(w)}
+
\Sigma_{s,\ell}^{(w)}
-
\Sigma_{t,k}^{(w)}
\mathbb E\!\left[
M_{t,s}^{(w)}
\middle|
Z_t^{(w)}=k,Z_s^{(w)}=\ell
\right]
\nonumber\\
&\quad
-
\mathbb E\!\left[
M_{t,s}^{(w)\top}
\middle|
Z_t^{(w)}=k,Z_s^{(w)}=\ell
\right]
\Sigma_{t,k}^{(w)}\\
&=\Sigma_{t,k}^{(w)}
+
\Sigma_{s,\ell}^{(w)}
-
\Xi^{(w)}(t,s;k,\ell)
-
\Xi^{(w)}(t,s;k,\ell)^\top,
\end{align}
which proves the stated expression for
\(\mathbf C_{\mathrm{dev}}^{(w)}(t,s;k,\ell)\).
\end{proof}

\section{Dynamic programming for transition quantities}
\label{app:dynamic_programming_transition}

%%%% BEFORE
%We describe how to compute the sense-transition quantities without enumerating
%all possible intermediate Sense paths. In particular, the sense-pair
%probabilities \(q^{(w)}(t,s;k,\ell)\) and the cross-period covariance
%quantities \(\Xi^{(w)}(t,s;k,\ell)\), and hence the deviation components
%\(\mathbf C_{\mathrm{dev}}^{(w)}(t,s;k,\ell)\), can be computed by dynamic
%programming.

We describe how to compute the sense-transition quantities without enumerating all possible intermediate sense paths. In particular, the sense-pair probabilities \(q^{(w)}(t,s;k,\ell)\) and the cross-period covariance quantities \(\Xi^{(w)}(t,s;k,\ell)\), and hence the deviation components \(\mathbf C_{\mathrm{dev}}^{(w)}(t,s;k,\ell)\), can be computed by dynamic
programming.

\textbf{Computational cost.}
We analyze the computational cost of computing non-adjacent quantities induced by the Markov chain after GMM fitting and adjacent period OT solving have already been completed, and hence the corresponding transport maps \(A_{t,k,\ell}^{(w)}\) and transition probabilities \(P_t^{(w)}(k,\ell)\) are available.
A direct pathwise computation requires enumerating latent component paths across periods, leading to an exponential increase in computational cost with the number of periods \(T\). We therefore exploit dynamic programming to aggregate contributions from intermediate latent components without explicit path enumeration.
For one word, let \(K=\max_t K_t^{(w)}\), and let \(d\) denote the embedding dimension. The memory costs below refer only to additional memory required during computation and exclude storage of the precomputed mixture parameters, adjacent period transport maps and transition probabilities, as well as the final outputs.

Considering the computation of all quantities indexed by \((t,s,k,\ell)\), the joint component probabilities \(q(t,s;k,\ell)\) can be obtained by either explicit enumeration of latent component paths or dynamic programming. A naive enumeration of all latent component paths yields the upper bound \(O(T^3K^T)\), whereas the dynamic programming recursion requires only \(O(T^2K^3)\). In terms of memory, naive path enumeration can be implemented with \(O(T)\) memory by generating one path at a time, while the dynamic programming recursion requires \(O(K^2)\) memory for the current transition-probability matrix.

For the computation of \(\Xi^{(w)}(t,s;k,\ell)\), naive enumeration of all latent component paths yields the upper bound \(O(T^3K^Td^3)\) in time and \(O(T+d^2)\) in memory. In contrast, the dynamic programming recursion requires \(O(T^2K^3d^3)\) time and \(O(K^2d^2)\) memory. The conditional weights \(r_{t,s}(j\mid k,\ell)\) can be computed in \(O(1)\) time for each \((t,s,k,\ell,j)\) without additional memory, provided that the recursion for \(R_{t,s}\) used to compute \(q(t,s;k,\ell)\) is carried out alongside the \(\Xi\) recursion.

Once \(\Xi^{(w)}(t,s;k,\ell)\) has been computed, the deviation component
\[
\mathbf C_{\mathrm{dev}}^{(w)}(t,s;k,\ell)
=
\Sigma_{t,k}^{(w)}
+
\Sigma_{s,\ell}^{(w)}
-
\Xi^{(w)}(t,s;k,\ell)
-
\Xi^{(w)}(t,s;k,\ell)^\top
\]
requires only \(O(d^2)\) time and \(O(d^2)\) memory for each \((t,s,k,\ell)\), yielding \(O(T^2K^2d^2)\) time overall. This is lower-order than the \(O(T^2K^3d^3)\) cost of computing \(\Xi\), and therefore the additional computational cost of obtaining \(\mathbf C_{\mathrm{dev}}\) is negligible in comparison. 

% A dense full-covariance GMM candidate with \(I\) EM iterations has a standard cost of \(O(I(nKd^2+Kd^3))\) per period. Grid search repeats this cost for its candidate configurations.
% Given the (T) fitted GMMs, for each adjacent period pair \((t,t+1)\), there are at most \(K^2\) Gaussian component pairs. Computing the Gaussian transport maps \(A_{t,k,\ell}^{(w)}\) and the corresponding transport costs \(C_{t,k,\ell}^{(w)}\) requires dense matrix square-root operations and therefore costs \(O(d^3)\) per component pair, yielding \(O(TK^2d^3)\) overall. The component-level coupling \(Q_t^{(w)}\) is then obtained by solving a \(K\times K\) linear program with at most \(K^2\) decision variables. Thus, the total cost of the adjacent period OT stage is \(O(TK^2d^3)\), in addition to the chosen discrete-transport solver's cost for \(T-1\) problems with at most \(K^2\) variables each.
% Constructing the mean basis from \(C_{\mathrm{mean}}^{(w)}\) requires \(O(TK^2d^2)\) operations for aggregating the component-level matrices over adjacent period pairs, followed by \(O(d^3)\) for the eigendecomposition of the resulting \(d\times d\) matrix. Thus, the overall cost is \(O(TK^2d^2+d^3)\).
% For each pair \((t,s)\), computing \(d_{\mathrm{tr}}(t,s)\) by summing \(q(t,s;k,\ell)\operatorname{Tr} C^{(w)}(t,s;k,\ell)\) over component pairs requires \(O(K^2d)\) operations. For each mode \(r\), computing \(d_r(t,s)\) by summing \(q(t,s;k,\ell)u_r^\top C^{(w)}(t,s;k,\ell)u_r\) over component pairs requires \(O(K^2d^2)\) operations.

\begin{proposition}[Dynamic programming recursions]
\label{prop:dynamic_programming_transition}
For \(t\le s\), let
\[
R_{t,s}^{(w)}(k,\ell)
:=
\Pr\!\left(
Z_s^{(w)}=\ell
\mid
Z_t^{(w)}=k
\right).
\]
Then
\begin{equation}
R_{t,t}^{(w)}=I_{K_t^{(w)}},
\qquad
R_{t,r+1}^{(w)}
=
R_{t,r}^{(w)}P_r^{(w)},
\quad r=t,\ldots,s-1,
\label{eq:dp_R_recursion}
\end{equation}
and hence
\begin{equation}
q^{(w)}(t,s;k,\ell)
=
\pi_{t,k}^{(w)}R_{t,s}^{(w)}(k,\ell).
\label{eq:dp_q_recursion}
\end{equation}

For \(t<s\), define the backward smoothing probability
\[
r_{t,s}^{(w)}(j\mid k,\ell)
:=
\Pr\!\left(
Z_{s-1}^{(w)}=j
\mid
Z_t^{(w)}=k,Z_s^{(w)}=\ell
\right).
\]
Then, whenever \(R_{t,s}^{(w)}(k,\ell)>0\),
\begin{equation}
r_{t,s}^{(w)}(j\mid k,\ell)
=
\frac{
R_{t,s-1}^{(w)}(k,j)P_{s-1}^{(w)}(j,\ell)
}{
R_{t,s}^{(w)}(k,\ell)
}.
\label{eq:dp_backward_smoothing}
\end{equation}
The recursion for \(\Xi^{(w)}\) is initialized by
\begin{equation}
\Xi^{(w)}(t,t;k,\ell)
=
\mathbf 1_{\{k=\ell\}}\Sigma_{t,k}^{(w)},
\label{eq:dp_Xi_initial}
\end{equation}
and, for \(s>t\),
\begin{equation}
\Xi^{(w)}(t,s;k,\ell)
=
\sum_{j=1}^{K_{s-1}^{(w)}}
r_{t,s}^{(w)}(j\mid k,\ell)
\,
\Xi^{(w)}(t,s-1;k,j)
\left(
A_{s-1,j,\ell}^{(w)}
\right)^\top.
\label{eq:dp_Xi_recursion}
\end{equation}
This recursion is evaluated for positive-mass endpoint pairs. Zero-weight terms in the sum are omitted, and we set \(\Xi^{(w)}(t,s;k,\ell)=0\) for zero-mass endpoint pairs.

\end{proposition}

\begin{proof}
The recursion for \(R_{t,s}^{(w)}\) follows from the Markov property. For
\(r=t,\ldots,s-1\),
\begin{align}
R_{t,r+1}^{(w)}(k,\ell)
&=
\Pr\!\left(
Z_{r+1}^{(w)}=\ell
\mid
Z_t^{(w)}=k
\right)
\nonumber\\
&=
\sum_{j=1}^{K_r^{(w)}}
\Pr\!\left(
Z_r^{(w)}=j
\mid
Z_t^{(w)}=k
\right)
\Pr\!\left(
Z_{r+1}^{(w)}=\ell
\mid
Z_r^{(w)}=j
\right)
\nonumber\\
&=
\sum_{j=1}^{K_r^{(w)}}
R_{t,r}^{(w)}(k,j)P_r^{(w)}(j,\ell),
\label{eq:dp_R_proof}
\end{align}
which gives \eqref{eq:dp_R_recursion}. Moreover,
\begin{align}
q^{(w)}(t,s;k,\ell)
&=
\Pr\!\left(
Z_t^{(w)}=k
\right)
\Pr\!\left(
Z_s^{(w)}=\ell
\mid
Z_t^{(w)}=k
\right)
\nonumber\\
&=
\pi_{t,k}^{(w)}R_{t,s}^{(w)}(k,\ell),
\label{eq:dp_q_proof}
\end{align}
which proves \eqref{eq:dp_q_recursion}.

Next, by the definition of conditional probability,
\begin{align}
r_{t,s}^{(w)}(j\mid k,\ell)
&=
\frac{
\Pr\!\left(
Z_{s-1}^{(w)}=j
\mid
Z_t^{(w)}=k
\right)
\Pr\!\left(
Z_s^{(w)}=\ell
\mid
Z_{s-1}^{(w)}=j,Z_t^{(w)}=k
\right)
}{
\Pr\!\left(
Z_s^{(w)}=\ell
\mid
Z_t^{(w)}=k
\right)
}
\nonumber\\
&=
\frac{
R_{t,s-1}^{(w)}(k,j)
P_{s-1}^{(w)}(j,\ell)
}{
R_{t,s}^{(w)}(k,\ell)
},
\label{eq:dp_backward_smoothing_proof}
\end{align}
where the final equality follows from the Markov property.

It remains to establish the recursion for \(\Xi^{(w)}\). From the definition
of \(M_{t,s}^{(w)}\) in \eqref{eq:app_def_M_ts}, for \(s>t\),
\begin{equation}
M_{t,s}^{(w)}
=
M_{t,s-1}^{(w)}
\left(
A_{s-1,Z_{s-1}^{(w)},Z_s^{(w)}}^{(w)}
\right)^\top.
\label{eq:dp_M_recursion}
\end{equation}

For the initial case, we use the empty-product convention
\(M_{t,t}^{(w)}:=I_d\), so that
\[
\Xi^{(w)}(t,t;k,k)
=
\Sigma_{t,k}^{(w)}.
\]

For \(s>t\), conditioning on \(Z_{s-1}^{(w)}\) gives
\begin{align}
&
\mathbb E\!\left[
M_{t,s}^{(w)}
\mid
Z_t^{(w)}=k,Z_s^{(w)}=\ell
\right]
\nonumber\\
&=
\sum_{j=1}^{K_{s-1}^{(w)}}
r_{t,s}^{(w)}(j\mid k,\ell)
\mathbb E\!\left[
M_{t,s-1}^{(w)}
\middle|
Z_t^{(w)}=k,Z_{s-1}^{(w)}=j,Z_s^{(w)}=\ell
\right]
\left(
A_{s-1,j,\ell}^{(w)}
\right)^\top.
\label{eq:dp_M_conditioning}
\end{align}
Since \(M_{t,s-1}^{(w)}\) depends only on
\(Z_t^{(w)},\ldots,Z_{s-1}^{(w)}\), the Markov property gives
\begin{align}
&
\mathbb E\!\left[
M_{t,s-1}^{(w)}
\mid
Z_t^{(w)}=k,Z_{s-1}^{(w)}=j,Z_s^{(w)}=\ell
\right]
\nonumber\\
&\qquad=
\mathbb E\!\left[
M_{t,s-1}^{(w)}
\mid
Z_t^{(w)}=k,Z_{s-1}^{(w)}=j
\right].
\label{eq:dp_M_markov}
\end{align}
Multiplying \eqref{eq:dp_M_conditioning} by
\(\Sigma_{t,k}^{(w)}\) yields
\begin{align}
\Xi^{(w)}(t,s;k,\ell)
&=
\sum_{j=1}^{K_{s-1}^{(w)}}
r_{t,s}^{(w)}(j\mid k,\ell)
\,
\Xi^{(w)}(t,s-1;k,j)
\left(
A_{s-1,j,\ell}^{(w)}
\right)^\top,
\label{eq:dp_Xi_proof}
\end{align}
which proves \eqref{eq:dp_Xi_recursion}.
\end{proof}

\section{Statistical recovery}
\label{appx:statistical_recovery}

This appendix studies how estimation errors propagate through the construction of the semantic geometry and establishes the error rates for the second-moment operators and semantic distances under the Gaussian specialization given in Theorem~\ref{thm:distance_recovery}.
%%%% BEFORE
%This appendix studies how estimation errors propagate through the construction of the semantic geometry, and establish the error rates for the second-moment operators and semantic distances under the Gaussian specialization, which are given in Theorem~\ref{thm:distance_recovery}.
Figure~\ref{fig:geometry_construction_pipeline} illustrates the construction pipeline from estimated sense distributions and prevalences to the semantic distances.
Figure~\ref{fig:geometry_error_propagation} summarizes the error propagation from the estimated Gaussian mixture parameters to the estimated semantic geometry. We first establish the estimation rate for the Gaussian mixture parameters and then propagate these errors through each step of the construction to obtain the convergence rate of the second-moment operators and the induced distances.

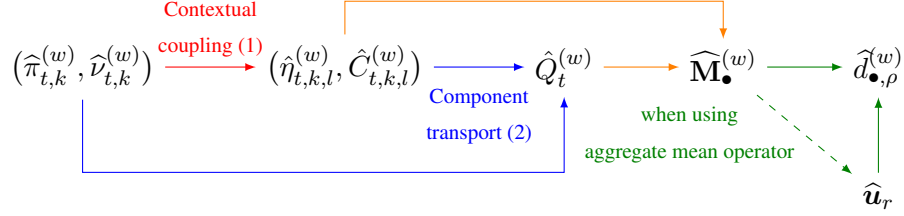
\begin{figure}[t]
\centering
\begin{tikzpicture}[
    node distance=1.0cm,
    >=latex,
    every node/.style={align=center}
]

\node (a)
{\(\bigl(\widehat\pi_{t,k}^{(w)},\widehat\nu_{t,k}^{(w)}\bigr)\)};

\node (b) [right=1.2cm of a]
{\(\bigl(\hat\eta_{t,k,l}^{(w)},\hat C_{t,k,l}^{(w)}\bigr)\)};

\node (c) [right=1.2cm of b]
{\(\hat Q_t^{(w)}\)};

\node (d) [right=of c]
{\(\widehat{\mathbf M}_{\bullet}^{(w)}\)};

\node (e) [right=of d]
{\(\widehat d_{\bullet,\rho}^{(w)}\)};

\node (f) [below=of e]
{\(\widehat{\bm u}_r\)};

\draw[->, red] (a) -- node[above]{\scriptsize Contextual \\ \scriptsize coupling~(\ref{eq:pairwise_distribution_coupling})} (b);
\draw[->, blue] (b) -- node[below=0.15cm]{\scriptsize Component \\ \scriptsize transport~(\ref{eq:sense_transition_distribution})}  (c);
\draw[->, blue] (a) |- ([yshift=-1.0cm]c.south) -- (c.south);
\draw[->, orange] (c) -- (d); 
\draw[->, orange] (b) |- ([yshift=0.5cm]d.north) -- (d.north);
\draw[->, dashed, green!50!black] (d) -- node[left]{\scriptsize when using \\ \scriptsize aggregate mean operator} (f);
\draw[->,green!50!black] (d) -- (e);
\draw[->, green!50!black] (f) -- (e);

\end{tikzpicture}
\caption{Construction pipeline of the estimated semantic geometry.
The pipeline maps estimated sense distributions to geometric representations.
\(\bullet\in\{\mathrm{full},\mathrm{mean},\mathrm{dev}\},
\rho\in\{\mathrm{tr},1,\ldots,d\}\).}
\label{fig:geometry_construction_pipeline}
\end{figure}

\begin{figure}[t]
\centering

\begin{tikzpicture}[
    node distance=0.8cm,
    >=latex,
    every node/.style={align=center}
]

\node (a)
{\(\bigl(\widehat\pi_{t,k}^{(w)}, \widehat\mu_{t,k}^{(w)},
\widehat\Sigma_{t,k}^{(w)}\bigr)\)};

\node (b) [right=of a]
{\(\bigl(\widehat A_{t,k,l}^{(w)},
\widehat C_{t,k,l}^{(w)}\bigr)\)};

\node (c) [right=of b]
{\(\widehat Q_t^{(w)}\)};

\node (d) [right=of c]
{\(\widehat{\mathbf M}_{\bullet}^{(w)}\)};

\node (e) [right=of d]
{\(\widehat d_{\bullet,\rho}^{(w)}\)};

\node (f) [below=of e]
{\(\widehat{\bm u}_r\)};

\draw[->, red]
(a) -- node[above=0.2cm]{\scriptsize Lemma~\ref{lem:gaussian_coupling_perturbation}} (b);

\draw[->, blue]
(b) -- node[above=0.2cm]{\scriptsize Lemma~\ref{lem:sense_transport_perturbation}} (c);

\draw[->, orange]
(c) -- node[above=0.2cm]{\scriptsize Lemma~\ref{lem:gaussian_second_moment_operator_perturbation}} (d);

\draw[->, blue]
(a) |- ([yshift=-0.5cm]c.south) -- (c.south);

\draw[->,  dashed, orange]
(a) |- ([yshift=0.5cm]d.north) -- (d.north);

\draw[->, orange]
(b) |- ([yshift=0.5cm]d.north) -- (d.north);

\draw[->, dashed, green!50!black]
(d) -- (f);

\draw[->,green!50!black]
(d) -- (e);

\draw[->, green!50!black]
(f) -- (e);

% annotations
\node[draw, rounded corners, font=\scriptsize, below=0.8cm of a]
(ann_a) at ([xshift=-0.4cm]a.south)
{Lemma~\ref{lem:gmm_parameter_recovery}\\ error \(O_p(n^{-1/2})\)};

\node[draw, rounded corners, font=\scriptsize, below=0.8cm of d] (ann_d)
{Theorem~\ref{thm:distance_recovery}\\ error \(O_p(n^{-1/2})\)};

% annotation arrows
\draw[->, gray]
(ann_a.north) -- ([xshift=-0.4cm]a.south);

\draw[->, gray]
(ann_d.north) -- (d.south);
\draw[->, gray]
(ann_d.north) -- (e.south west);

\end{tikzpicture}
%%%% BEFORE 
%\caption{Error propagation and recovery guarantees for the estimated semantic geometry.
%\(\bullet\in\{\mathrm{full},\mathrm{mean},\mathrm{dev}\},
%\rho\in\{\mathrm{tr},1,\ldots,d\}\).}
\caption{Error propagation and recovery guarantees for the estimated semantic geometry. \(\bullet\in\{\mathrm{full},\mathrm{mean},\mathrm{dev}\}, \rho\in\{\mathrm{tr},1,\ldots,R^{(w)}\}\).}
\label{fig:geometry_error_propagation}
\end{figure}
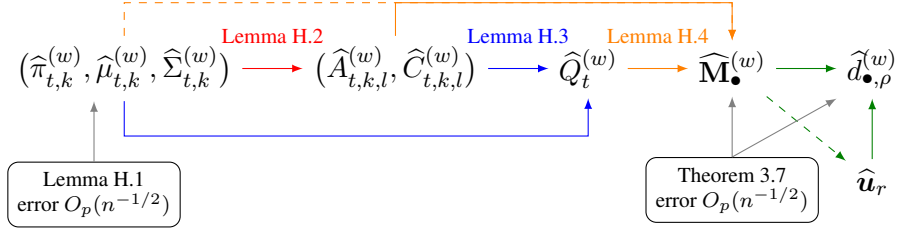

\paragraph{Recovery of Gaussian mixture parameters.}
We assume throughout the theoretical analysis that the true number of mixture components is known. This assumption does not impose a substantive restriction. When a known finite upper bound $K_{\max}$ on the number of components is available, an estimate $\widehat K_t^{(w)}$ can be obtained by a penalized likelihood model-selection procedure.
Under the regularity conditions in Definition~\ref{def:gmm_class}, the model-selection consistency established by \citet[Theorem~3]{huang2017model} implies that
\[
\Pr\!\left(
\widehat K_t^{(w)}=K_t^{(w)}
\right)
\to 1
\qquad(n\to\infty).
\]
Thus, the true component number is recovered with probability tending to one as the sample size increases. We therefore condition the subsequent analysis on the event of correct component-number recovery and focus on the estimation error of the Gaussian mixture parameters and its propagation through the subsequent stages.

%The recovery theorem treats the component counts \(K_t^{(w)}\) as known and obtains parametric recovery of the Gaussian parameters from the penalized estimator of \citet{chen2009inference}, after aligning component labels. Consistency of particular penalized model-selection procedures is available under additional conditions \citep{huang2017model}, but is not required for the theorem. The bounded BIC selection used in the lexical experiments is a separate empirical model-selection step. The recovery theorem does not establish correct component-number selection for this procedure. If the component count is misspecified, the decomposition remains exact for the fitted process, but its allocation between center movement and within-component variation can depend on that count.

%%%% BEFORE
%The bounded BIC selection used in the lexical experiments is a separate empirical model-selection step.

\paragraph{Gaussian mixture parameter estimation.}
Taking the true component numbers \(K_t^{(w)}\) as known, we next consider the estimation of the Gaussian mixture parameters. For a Gaussian mixture distribution
\[
\gamma
=
\sum_{k=1}^{K}\pi_k\nu_k,
\qquad
\nu_k=\mathcal N(m_k,\Sigma_k),
\]
define
\[
\theta_k
=
\left(
m_k^\top,
\sigma_k(i,j):
i=1,\ldots,d,\;
j=1,\ldots,i
\right)^\top,
\]
where $\sigma_k(i,j)$ denotes the $(i,j)$-th entry of $\Sigma_k$. The full Gaussian mixture parameter vector is
%%%% BEFORE 
%\[
%\theta
%=
%(\pi_1,\ldots,\pi_K,
%\theta_1^\top,\ldots,\theta_K^\top)^\top.
%\]
\[
\theta
=
(\pi_1,\ldots,\pi_K,
\theta_1^\top,\ldots,\theta_K^\top)^\top.
\]

%%%% BEFORE
%For asymptotic inference, use the nonredundant coordinates
%\[
%\vartheta
%=
%(\pi_1,\ldots,\pi_{K-1},
%\theta_1^\top,\ldots,\theta_K^\top)^\top,
%\qquad
%\pi_K=1-\sum_{k=1}^{K-1}\pi_k.
%\]
%For \(K=1\), there are no free mixing-weight coordinates. The full vector \(\theta\) is retained for the parameter-error bounds below.

The full vector \(\theta\) is retained for the componentwise parameter-error bounds below.

%For the known number of components, we estimate the parameters using the penalized maximum likelihood procedure of~\cite{chen2009inference}. Specifically, let
For the theoretical guarantee with fixed component counts, we use the penalized maximum likelihood estimator of \citet{chen2009inference}. To specify it, let
\[
p\ell(\theta)
=
\ell_n(\theta)+p_n(\theta),
\]
where $\ell_n(\theta)$ denotes the ordinary log-likelihood based on the observed sample, and
\[
p_n(\theta)
=
-\frac1n\sum_{k=1}^{K}
\left(
\operatorname{tr}(\Sigma_k^{-1})
+
\log|\Sigma_k|
\right).
\]
%\[
%p_n(\theta)
%=
%-a_n\sum_{k=1}^{K}
%\left(
%\operatorname{tr}(S_X\Sigma_k^{-1})
%+
%\log|\Sigma_k|
%\right),
%\qquad a_n=n^{-1},
%\]
%Here \(S_X\) is the sample covariance of the observations used to fit the mixture. The estimator is
Then, the estimator is
\[
\widehat\theta
\in
\arg\max_{\theta}p\ell(\theta).
\]

%Under the regularity conditions in Definition~\ref{def:gmm_class}, Theorem~2 of~\cite{chen2009inference} gives the asymptotic normality
%%%% BEFORE 
%Under the penalty and regularity conditions of \citet{chen2009inference}, their Theorem~2 gives the asymptotic normality
%\begin{equation}
%\sqrt n
%\left(
%\widehat\theta_t^{(w)}
%-
%\theta_t^{(w)}
%\right)
%\xrightarrow{d}
%\mathcal N\!\left(
%0,
%I(\theta_t^{(w)})^{-1}
%\right),
%\label{eq:gmm_parameter_%asymptotic_normality}
%\end{equation}
%where $I(\theta_t^{(w)})$ denotes the Fisher information matrix, which is positive definite under the stated regularity conditions.

%%%% BEFORE
%Under the penalty and regularity conditions of \citet{chen2009inference}, their Theorem~2 gives, after aligning component labels,
%\begin{equation}
%\sqrt n
%\left(
%\widehat\vartheta_t^{(w)}-\vartheta_t^{(w)}
%\right)
%\xrightarrow{d}
%\mathcal N\!\left(
%0,I_{\vartheta}(\vartheta_t^{(w)})^{-1}
%\right),
%\label{eq:gmm_parameter_asymptotic_normality}
%\end{equation}
%where \(I_{\vartheta}\) is the Fisher information matrix in these nonredundant coordinates and is positive definite under the stated regularity conditions. The full vector \(\theta_t^{(w)}\) is an affine function of \(\vartheta_t^{(w)}\), so it inherits the \(O_p(n^{-1/2})\) parameter-error rate.
%Therefore, for any $\varepsilon>0$, there exists $M<\infty$ such that
%\[
%\Pr\!\left(
%\sqrt n
%\left\|
%\widehat\theta_t^{(w)}-\theta_t^{(w)}
%\right\|
%\leq M
%\right)
%\geq 1-\varepsilon
%\]
%for all sufficiently large $n$. Equivalently,
%\[
%\left\|
%\widehat\theta_t^{(w)}-\theta_t^{(w)}
%\right\|
%=
%O_p(n^{-1/2}).
%\]

Under the regularity conditions in Definition~\ref{def:gmm_class}, Theorem~2 of~\cite{chen2009inference} gives the asymptotic normality in a nonredundant local parameterization. In particular, after aligning component labels, the full constrained parameter vector satisfies
\[
\left\|
\widehat\theta_t^{(w)}-\theta_t^{(w)}
\right\|
=
O_p(n^{-1/2}).
\]

To quantify the parameter estimation errors, define
\[
\begin{aligned}
E_{\mu}^{(w)}
&:=
\max_{t,k}
\left\|
\widehat\mu_{t,k}^{(w)}
-
\mu_{t,k}^{(w)}
\right\|_2,\\
E_{\Sigma}^{(w)}
&:=
\max_{t,k}
\left\|
\widehat\Sigma_{t,k}^{(w)}
-
\Sigma_{t,k}^{(w)}
\right\|_2,\\
E_{\pi}^{(w)}
&:=
\max_{t,k}
\left|
\widehat\pi_{t,k}^{(w)}
-
\pi_{t,k}^{(w)}
\right|.
\end{aligned}
\]

\begin{lemma}[Gaussian mixture parameter recovery]
\label{lem:gmm_parameter_recovery}
%%%% BEFORE
%Under the regularity conditions in Definition~\ref{def:gmm_class} and for known \(K\) penalized estimator above \citet{chen2009inference},
For known \(K\), under Definition~\ref{def:gmm_class} and the penalty and regularity conditions of \citet{chen2009inference}, the penalized estimator described above satisfies
\[
E_{\mu}^{(w)}
,
E_{\Sigma}^{(w)}
,
E_{\pi}^{(w)}
=
O_p(n^{-1/2}).
\]
\end{lemma}

\begin{proof}
By the definition of $\theta_t^{(w)}$ and the symmetry of the covariance
matrices,
\[
\begin{aligned}
\left\|
\widehat\theta_t^{(w)}-\theta_t^{(w)}
\right\|^2
\geq{}&
\sum_{k=1}^{K_t^{(w)}}
\left|
\widehat\pi_{t,k}^{(w)}
-
\pi_{t,k}^{(w)}
\right|^2
\\
&+
\sum_{k=1}^{K_t^{(w)}}
\left\|
\widehat\mu_{t,k}^{(w)}
-
\mu_{t,k}^{(w)}
\right\|_2^2
\\
&+
\frac{1}{2}
\sum_{k=1}^{K_t^{(w)}}
\left\|
\widehat\Sigma_{t,k}^{(w)}
-
\Sigma_{t,k}^{(w)}
\right\|_2^2.
\end{aligned}
\]
Thus, for each $t$,
\[
\max_k
\left|
\widehat\pi_{t,k}^{(w)}
-
\pi_{t,k}^{(w)}
\right|
,
\max_k
\left\|
\widehat\mu_{t,k}^{(w)}
-
\mu_{t,k}^{(w)}
\right\|_2
,
\max_k
\left\|
\widehat\Sigma_{t,k}^{(w)}
-
\Sigma_{t,k}^{(w)}
\right\|_2
=
O_p(n^{-1/2}).
\]
Since the number of periods is fixed, taking the maximum over $t$ gives
\[
E_{\pi}^{(w)}
,
E_{\mu}^{(w)}
,
E_{\Sigma}^{(w)}
=
O_p(n^{-1/2}).
\]
\end{proof}

\paragraph{Perturbation of Gaussian coupling maps and costs.}

Define the perturbation errors of the Gaussian coupling maps and transport
costs by
\[
E_A^{(w)}
:=
\max_{t,k,\ell}
\left\|
\widehat A_{t,k,\ell}^{(w)}
-
A_{t,k,\ell}^{(w)}
\right\|_2,
\]
and
\[
E_C^{(w)}
:=
\max_{t,k,\ell}
\left|
\widehat C_{t,k,\ell}^{(w)}
-
C_{t,k,\ell}^{(w)}
\right|.
\]

\begin{lemma}[Perturbation of Gaussian coupling maps and costs]
\label{lem:gaussian_coupling_perturbation}
Under the regularity conditions in Definition~\ref{def:gmm_class},
as $n\to\infty$ with \(E_\mu^{(w)}, E_\Sigma^{(w)}\to0\),
\[
E_A^{(w)}
=
O\!\left(E_{\Sigma}^{(w)}\right),
\]
and
\[
E_C^{(w)}
=
O\!\left(
E_{\mu}^{(w)}
+
E_{\Sigma}^{(w)}
\right).
\]
\end{lemma}

\begin{proof}
We use the following perturbation bounds. For matrices $X,Y$ and
perturbations $\Delta X,\Delta Y$ satisfying
\[
\|\Delta X\|_2,\|\Delta Y\|_2\leq E,
\]
we have
\[
\|(X+\Delta X)(Y+\Delta Y)-XY\|_2
\leq
2\max\{\|X\|_2,\|Y\|_2\}E+o(E).
\]
For $X\succ0$ and $\|\Delta X\|_2\leq E$,
\[
\begin{aligned}
\|(X+\Delta X)^{-1}-X^{-1}\|_2
&\leq
\frac{1}{\lambda_{\min}(X)^2}E+o(E),\\
\|(X+\Delta X)^{1/2}-X^{1/2}\|_2
&\leq
\frac{1}{2\lambda_{\min}(X)^{1/2}}E+o(E),\\
\|(X+\Delta X)^{-1/2}-X^{-1/2}\|_2
&\leq
\frac{1}{2\lambda_{\min}(X)^{3/2}}E+o(E).
\end{aligned}
\]

Under Definition~\ref{def:gmm_class},
\[
\lambda_{\min}(\Sigma_{t,k}^{(w)})>0.
\]
Applying the above perturbation bounds successively to
\[
A_{t,k,\ell}^{(w)}
=
\left(\Sigma_{t,k}^{(w)}\right)^{-1/2}
\Bigl[
\left(\Sigma_{t,k}^{(w)}\right)^{1/2}
\Sigma_{t+1,\ell}^{(w)}
\left(\Sigma_{t,k}^{(w)}\right)^{1/2}
\Bigr]^{1/2}
\left(\Sigma_{t,k}^{(w)}\right)^{-1/2}.
\]
gives
\[
\left\|
\widehat A_{t,k,\ell}^{(w)}
-
A_{t,k,\ell}^{(w)}
\right\|_2
=
O\!\left(
\left\|
\widehat\Sigma_{t,k}^{(w)}
-
\Sigma_{t,k}^{(w)}
\right\|_2
+
\left\|
\widehat\Sigma_{t+1,\ell}^{(w)}
-
\Sigma_{t+1,\ell}^{(w)}
\right\|_2
\right).
\]
Since $T$ and $K_t^{(w)}$ are fixed, taking the maximum over $t,k,\ell$ preserves the order, and
hence
\[
E_A^{(w)}
=
O\!\left(E_{\Sigma}^{(w)}\right).
\]

Likewise, consider the transport cost
\[
\begin{aligned}
C_{t,k,\ell}^{(w)}
={}&
\left\|
\mu_{t,k}^{(w)}-\mu_{t+1,\ell}^{(w)}
\right\|^2
\\
&+
\operatorname{Tr}\!\left(
\Sigma_{t,k}^{(w)}
+
\Sigma_{t+1,\ell}^{(w)}
-
2
\left[
%(\Sigma_{t,k}^{(w)})^{1/2}
%\Sigma_{s,\ell}^{(w)}
%(\Sigma_{t,k}^{(w)})^{1/2}
(\Sigma_{t,k}^{(w)})^{1/2}
\Sigma_{t+1,\ell}^{(w)}
(\Sigma_{t,k}^{(w)})^{1/2}
\right]^{1/2}
\right).
\end{aligned}
\]
The squared mean term is Lipschitz in the mean parameters on the bounded
parameter space, while the covariance terms are controlled by the above
matrix perturbation bounds. In particular, since $d$ is fixed,
the trace terms can be bounded in terms of the corresponding spectral
norm perturbations. Consequently,
\[
\begin{aligned}
\left|
\widehat C_{t,k,\ell}^{(w)}
-
C_{t,k,\ell}^{(w)}
\right|
=
O\!\Big(
&
\left\|
\widehat\mu_{t,k}^{(w)}
-
\mu_{t,k}^{(w)}
\right\|_2
+
\left\|
\widehat\mu_{t+1,\ell}^{(w)}
-
\mu_{t+1,\ell}^{(w)}
\right\|_2
\\
&+
\left\|
\widehat\Sigma_{t,k}^{(w)}
-
\Sigma_{t,k}^{(w)}
\right\|_2
+
\left\|
\widehat\Sigma_{t+1,\ell}^{(w)}
-
\Sigma_{t+1,\ell}^{(w)}
\right\|_2
\Big).
\end{aligned}
\]
Again, since $T$ and $K_t^{(w)}$ are fixed, taking the maximum over $t,k,\ell$ preserves the order. Thus,
\[
E_C^{(w)}
=
O\!\left(
E_{\mu}^{(w)}
+
E_{\Sigma}^{(w)}
\right).
\]
\end{proof}

\paragraph{Perturbation of the sense transport plans.}

We next consider the perturbation of the optimal sense transport plans
induced by estimation errors in the sense prevalences and the transport
costs. Define
\[
E_Q^{(w)}
:=
\max_t
\left\|
\widehat Q_t^{(w)}-Q_t^{(w)}
\right\|_{\mathrm F}.
\]

The perturbation $E_\pi^{(w)}$ corresponds to the perturbation of the
right-hand side of the equality constraints in
\eqref{eq:sense_transition_distribution}, while $E_C^{(w)}$ corresponds
to the perturbation of its objective coefficients.

The following lemma is based on a standard basis-stability argument for linear programs. The cost perturbation $E_C^{(w)}$ does not appear in the
leading term of the bound, but it is essential for ensuring that the support of the optimal transport plan $Q_t^{(w)}$ remains unchanged under small perturbations. Once the support is fixed, the perturbation of $Q_t^{(w)}$ is determined solely by the perturbation of the marginal prevalences.

\begin{lemma}[Perturbation of the sense transport plans]
\label{lem:sense_transport_perturbation}
%%%% BEFORE
%Suppose that, the optimal sense transport plan $Q_t^{(w)}$ is unique and satisfies
Suppose that the optimal sense transport plan $Q_t^{(w)}$ is unique and satisfies
\[
\left|\operatorname{supp}\left(Q_t^{(w)}\right)\right|
=
K_t^{(w)}+K_{t+1}^{(w)}-1.
\]
Then, as $n\to\infty$ with \(E_C^{(w)}, E_\pi^{(w)}\to0\),
\begin{equation}
E_Q^{(w)}
=
O\!\left(E_\pi^{(w)}\right).
\label{eq:sense_transport_perturbation}
\end{equation}
\end{lemma}

\begin{proof}
For each consecutive period pair $(t,t+1)$, the equality constraints in
\eqref{eq:sense_transition_distribution} have rank
\[
K_t^{(w)}+K_{t+1}^{(w)}-1,
\]
since one of the marginal constraints is redundant. Let
$b_t^{(w)}$ denote the right-hand side obtained by removing one redundant
marginal constraint and stacking the remaining marginal probabilities.

Let $A_t^{(w)}$ denote the corresponding constraint matrix, and let
$q_t^{(w)}$ denote the vector obtained by stacking all entries of
$Q_t^{(w)}$. Then the equality constraints can be written as
\[
A_t^{(w)}q_t^{(w)}=b_t^{(w)}.
\]
By the support condition,
\[
\left|\operatorname{supp}\left(Q_t^{(w)}\right)\right|
=
K_t^{(w)}+K_{t+1}^{(w)}-1,
\]
the columns of $A_t^{(w)}$ corresponding to the nonzero entries of
$Q_t^{(w)}$ form a nonsingular square submatrix
\[
B_t^{(w)}
\in
\mathbb R^{(K_t^{(w)}+K_{t+1}^{(w)}-1)
\times
(K_t^{(w)}+K_{t+1}^{(w)}-1)}.
\]
Thus, the nonzero entries of $q_t^{(w)}$ are given by
\[
q_{t,B}^{(w)}
=
\left(B_t^{(w)}\right)^{-1}b_t^{(w)}.
\]

The uniqueness of $Q_t^{(w)}$ together with the support condition implies
that the corresponding basic feasible solution is nondegenerate and
that all nonbasic reduced costs are strictly positive. Therefore, for
sufficiently small $E_C^{(w)}$ and $E_\pi^{(w)}$, the same basis
$B_t^{(w)}$ remains optimal. Consequently, the basic entries of
$\widehat q_t^{(w)}$ are given by
\[
\widehat q_{t,B}^{(w)}
=
\left(B_t^{(w)}\right)^{-1}\widehat b_t^{(w)},
\]
and hence
\[
\widehat q_{t,B}^{(w)}-q_{t,B}^{(w)}
=
\left(B_t^{(w)}\right)^{-1}
\left(
\widehat b_t^{(w)}-b_t^{(w)}
\right).
\]

Since the nonbasic variables remain zero,
\[
\begin{aligned}
\left\|
\widehat Q_t^{(w)}-Q_t^{(w)}
\right\|_{\mathrm F}
&=
\left\|
\widehat q_{t,B}^{(w)}-q_{t,B}^{(w)}
\right\|_2 \\
&=
\left\|
\left(B_t^{(w)}\right)^{-1}
\left(
\widehat b_t^{(w)}-b_t^{(w)}
\right)
\right\|_2\\
&\leq
\left\|
\left(B_t^{(w)}\right)^{-1}
\right\|_2
\left\|
\widehat b_t^{(w)}-b_t^{(w)}
\right\|_2\\
&=
O\!\left(E_\pi^{(w)}\right).
\end{aligned}
\]
Taking the maximum over $t$ yields
\[
E_Q^{(w)}
=
O\!\left(E_\pi^{(w)}\right).
\]
\end{proof}

\paragraph{Perturbation of the induced second-moment operators.}

Finally, we propagate the preceding estimation errors to the induced
second-moment operators. For \(\bullet\in\{\mathrm{full},\mathrm{mean},\mathrm{dev}\}\), define
\[
E_{M,\bullet}^{(w)}
:=
\max_{t,s}
\left\|
\widehat{\mathbf M}_{\bullet}^{(w)}(t,s)
-
\mathbf M_{\bullet}^{(w)}(t,s)
\right\|_2.
\]

\begin{lemma}[Gaussian second-moment operator perturbation]
\label{lem:gaussian_second_moment_operator_perturbation}
As $n\to\infty$ with 
\[
E_Q^{(w)},\,
E_{\mu}^{(w)},\,
E_{\Sigma}^{(w)},\,
E_A^{(w)}
\to 0,
\]
we have
\[
E_{M,\mathrm{mean}}^{(w)}
=
O\!\left(
E_Q^{(w)}
+
E_{\mu}^{(w)}
\right),
\]
\[
E_{M,\mathrm{dev}}^{(w)}
=
O\!\left(
E_Q^{(w)}
+
E_{\Sigma}^{(w)}
+
E_A^{(w)}
\right),
\]
and
\[
E_{M,\mathrm{full}}^{(w)}
=
O\!\left(
E_Q^{(w)}
+
E_{\mu}^{(w)}
+
E_{\Sigma}^{(w)}
+
E_A^{(w)}
\right).
\]
\end{lemma}

\begin{proof}
By \eqref{eq:sense_transition_operator_decomposition}, the mean and deviation operators are finite sums of terms weighted by the sense-transition probabilities. For the mean component, $\mathbf C_{\mathrm{mean}}^{(w)}(t,s;k,\ell)$ is determined by the corresponding sense means. Therefore, we obtain
\[
E_{M,\mathrm{mean}}^{(w)}
=
O\!\left(
E_Q^{(w)}
+
E_\mu^{(w)}
\right).
\]

For the deviation component, the covariance terms contribute \(O(E_\Sigma^{(w)})\). Moreover, the cross-covariance terms involve \(\Sigma_{t,k}^{(w)}\) multiplied by conditional expectations of finite products of the coupling maps \(A^{(w)}\). Since the number of periods is fixed, the perturbation of these products is \(O(E_A^{(w)})\). Combining these perturbations with the perturbation of the transition probabilities gives
\[
E_{M,\mathrm{dev}}^{(w)}
=
O\!\left(
E_Q^{(w)}
+
E_\Sigma^{(w)}
+
E_A^{(w)}
\right).
\]

Finally, by \eqref{eq:mean_dev_operator_decomposition},
\[
E_{M,\mathrm{full}}^{(w)}
=
O\!\left(
E_Q^{(w)}
+
E_\mu^{(w)}
+
E_\Sigma^{(w)}
+
E_A^{(w)}
\right).
\]
\end{proof}

\paragraph{Proof of distance statistical recovery.}

The preceding results establish the estimation rate for the Gaussian mixture parameters and propagate it through the coupling construction, the sense transport plans, and the induced second-moment operators.
Combining these results with the second-moment operator-to-distance perturbation bound of \cite{ezoe2026multiscale} yields the desired distance recovery rates.

\begin{proof}[Proof of Theorem~\ref{thm:distance_recovery}]
Under Assumptions~\ref{ass:coupling_construction_uniqueness} and
\ref{ass:gaussian_estimation_regularity}, combining
Lemma~\ref{lem:gmm_parameter_recovery},
Lemma~\ref{lem:gaussian_coupling_perturbation},
Lemma~\ref{lem:sense_transport_perturbation}, and
Lemma~\ref{lem:gaussian_second_moment_operator_perturbation} yields
\[
E_{M,\bullet}^{(w)}
=
O_p\!\left((n^{(w)})^{-1/2}\right).
\]

We next relate the second-moment operator perturbation to the induced
distance perturbation. The same argument as in
\cite{ezoe2026multiscale} applies, without the need for an orthogonal
rotation between the population and estimated second-moment operators.

%For the shared mean basis, under Assumption~\ref{assump:separated_eigenvalues},
For the word-local mean-path basis, under Assumption~\ref{assump:separated_eigenvalues},
the same argument as in Proposition~C.4 of
\cite{ezoe2026multiscale} yields
\[
E_{u,r}^{(w)}
:=
\min_{\epsilon\in\{-1,1\}}
\left\|
\widehat{\bm u}_r^{(w)}
-
\epsilon\bm u_r^{(w)}
\right\|_2
\le
\frac{2^{3/2}T}{\delta}E_{M,\mathrm{mean}}^{(w)}
\]
for every \(r\in[R^{(w)}]\). Hence,
\[
E_{u,r}^{(w)}
=
O_p\!\left((n^{(w)})^{-1/2}\right),
\]
since \(T/\delta=O(1)\).

Moreover, the same argument as in Proposition~C.7 of
\cite{ezoe2026multiscale} yields
\[
\max_{t,s\in[T]}
\left|
\widehat d_{\bullet,\mathrm{tr}}^{(w)}(t,s)^2
-
d_{\bullet,\mathrm{tr}}^{(w)}(t,s)^2
\right|
\le
d E_{M,\bullet}^{(w)},
\]
and, for every \(r\in[R^{(w)}]\),
\[
\max_{t,s\in[T]}
\left|
\widehat d_{\bullet,r}^{(w)}(t,s)^2
-
d_{\bullet,r}^{(w)}(t,s)^2
\right|
\le
E_{M,\bullet}^{(w)}
+
2E_{u,r}^{(w)}
\max_{t,s\in[T]}
\|\mathbf M_{\bullet}^{(w)}(t,s)\|.
\]
Since the population quantities are fixed, combining these
bounds with
\[
E_{M,\bullet}^{(w)}
=
O_p\!\left((n^{(w)})^{-1/2}\right)
\]
gives
\[
\max_{t,s\in[T]}
\left|
\widehat d_{\bullet,\rho}^{(w)}(t,s)^2
-
d_{\bullet,\rho}^{(w)}(t,s)^2
\right|
=
O_p\!\left((n^{(w)})^{-1/2}\right)
\]
for \(\rho\in\{\mathrm{tr},1,\ldots,R^{(w)}\}\), which proves the theorem.
\end{proof}

%%%%APPENDIX:EXPERIMENT%%%%
\section{Supplementary experiments}
\label{app:experiments}

The supplementary experiments follow the progression of Section~\ref{sec:experiments}. The synthetic study tests finite-sample recovery of the Gaussian estimator. DWUG provides standard LSC validation and examines the mean--deviation decomposition. Janus isolates multi-period profile recovery and compositional coherence. The Court analysis documents the natural text corpus and shows how numerical transition and mode readouts lead to inspectable passages. Each subsection supplies construction details, robustness results, or audit evidence for a claim made in the main text.

\subsection{Synthetic mixture process and supplementary recovery}
\label{app:synthetic_details}

This experiment tests finite-sample recovery under the Gaussian conditions of Theorem~\ref{thm:distance_recovery}. It is a controlled check of the estimator rather than an additional application domain.

\noindent\textbf{Population and estimation.} We fix a three-period process in \(d=2\) with two components per period. Prevalences and centers are
\[
\begin{aligned}
\pi_1&=(.55,.45),& \mu_1&=\begin{pmatrix}-2.4&0\\2.2&0\end{pmatrix},\\
\pi_2&=(.40,.60),& \mu_2&=\begin{pmatrix}-1.9&.35\\2.7&-.25\end{pmatrix},\\
\pi_3&=(.28,.72),& \mu_3&=\begin{pmatrix}-1.3&.70\\3.1&.45\end{pmatrix}.
\end{aligned}
\]
The full covariance matrices are
\[
\begin{array}{c|cc}
t & \Sigma_{t,1} & \Sigma_{t,2}\\
\hline
1 & \left(\begin{smallmatrix}.30&0\\0&.30\end{smallmatrix}\right) & \left(\begin{smallmatrix}.35&0\\0&.25\end{smallmatrix}\right)\\
2 & \left(\begin{smallmatrix}1.60&.20\\.20&.40\end{smallmatrix}\right) & \left(\begin{smallmatrix}.45&-.15\\-.15&1.80\end{smallmatrix}\right)\\
3 & \left(\begin{smallmatrix}3.20&.35\\.35&.50\end{smallmatrix}\right) & \left(\begin{smallmatrix}.55&.20\\.20&3.50\end{smallmatrix}\right).
\end{array}
\]
Thus the population contains prevalence shifts, center displacement, and anisotropic within component change. Mean and deviation account for \(0.873\) and \(0.127\) of adjacent full path energy.

For each \(n\in\{100,200,400,800,1600\}\), we draw \(n\) independent usages per period and repeat the experiment \(50\) times. We fit a full-covariance GMM independently in every period with \(K=2\) fixed, covariance regularization \(10^{-6}\), five initializations, and no access to sampled component labels. Fixing \(K\) matches Assumption~\ref{ass:gaussian_estimation_regularity}. Because mixture labels are arbitrary, fitted components are aligned to the population components only when calculating errors.

\noindent\textbf{Errors and supplementary results.} We use
\[
E_\pi=\max_{t,k}|\widehat\pi_{t,k}-\pi_{t,k}|,\quad
E_\mu=\max_{t,k}\|\widehat\mu_{t,k}-\mu_{t,k}\|_2,\quad
E_\Sigma=\max_{t,k}\|\widehat\Sigma_{t,k}-\Sigma_{t,k}\|_2.
\]
For \(\bullet\in\{\mathrm{full},\mathrm{mean},\mathrm{dev}\}\), operator error is
\[
E_{M,\bullet}=\max_{t,s}\|\widehat{\mathbf M}_\bullet(t,s)-\mathbf M_\bullet(t,s)\|_2.
\]
Squared-distance error is the maximum absolute error over period pairs for either the trace or one of the two word-local modes:
\[
E_{d,\bullet}=\max_{t,s}\left|
\widehat d_{\bullet,\rho}(t,s)^2
-
d_{\bullet,\rho}(t,s)^2
\right|.
\]
The fitted mode distances use eigenvectors estimated from the fitted mean-path operator. Curves report means and normal approximation \(95\%\) confidence intervals across repetitions. Table~\ref{tab:app_synthetic_slopes} estimates slopes by ordinary least squares on the five log mean errors.

\begin{table*}[!htbp]
\centering
\caption{Synthetic finite-sample recovery. The last column reports mean error at \(n=1600\).}
\label{tab:app_synthetic_slopes}
\small
\setlength{\tabcolsep}{5pt}
\begin{tabular}{@{}lcc@{}}
\toprule
Quantity & Log--log slope [95\% CI] & Error at \(n=1600\)\\
\midrule
Prevalence \(E_\pi\) & \(-0.529\) [\(-0.599,-0.458\)] & \(0.017\)\\
Center \(E_\mu\) & \(-0.437\) [\(-0.565,-0.310\)] & \(0.135\)\\
Covariance \(E_\Sigma\) & \(-0.445\) [\(-0.542,-0.348\)] & \(0.337\)\\
Full operator \(E_{M,\mathrm{full}}\) & \(-0.532\) [\(-0.595,-0.469\)] & \(0.570\)\\
Mean operator \(E_{M,\mathrm{mean}}\) & \(-0.534\) [\(-0.606,-0.463\)] & \(0.581\)\\
Deviation operator \(E_{M,\mathrm{dev}}\) & \(-0.514\) [\(-0.654,-0.375\)] & \(0.101\)\\
Full trace distance & \(-0.532\) [\(-0.588,-0.476\)] & \(0.580\)\\
Full mode-1 distance & \(-0.540\) [\(-0.606,-0.473\)] & \(0.542\)\\
Full mode-2 distance & \(-0.561\) [\(-0.716,-0.406\)] & \(0.089\)\\
Mean trace distance & \(-0.542\) [\(-0.615,-0.468\)] & \(0.558\)\\
Deviation trace distance & \(-0.514\) [\(-0.687,-0.341\)] & \(0.099\)\\
\bottomrule
\end{tabular}
\end{table*}

The operator, trace-distance, and mode-distance intervals all contain the \(n^{-1/2}\) reference rate. Parameter errors also decrease with \(n\), with the center and covariance curves somewhat shallower over this finite range. Figure~\ref{fig:app_synthetic_distance_components} separates mean- and deviation-process distances.

\begin{figure*}[t]
\centering
\includegraphics[width=0.42\textwidth]{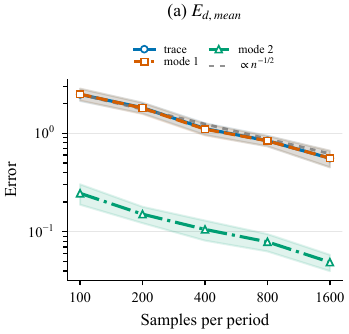}\hfill
\includegraphics[width=0.42\textwidth]{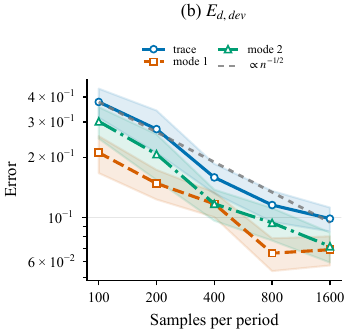}
\caption{Squared-distance recovery for the (a) mean and (b) deviation processes. Trace and mode distances use the corresponding operator with the estimated word-local basis. Bands are \(95\%\) confidence intervals over \(50\) repetitions.}
\label{fig:app_synthetic_distance_components}
\end{figure*}

\noindent\textbf{Pure mechanism checks.} We test the decomposition with three populations in which only centers, covariances, or prevalences change. Center movement produces mean energy. Changing prevalences between fixed components also transports mass between their centers and therefore produces mean energy. The one-component covariance process produces deviation energy. With \(n=1600\) usages per period and \(20\) repetitions, Table~\ref{tab:app_synthetic_mechanisms} shows that CUSP assigns \(99.89\%\), \(98.95\%\), and \(99.95\%\) of path energy to the respective planted mechanisms. The remaining \(1.05\%\) mean allocation in the covariance-only process is finite-sample estimation error.

%%%% BEFORE
%Table~\ref{tab:app_synthetic_mechanisms} shows that CUSP assigns \(99.9\%\), \(99.0\%\), and \(100.0\%\) of path energy to the respective planted mechanisms. The remaining \(1.0\%\) mean allocation in the covariance-only process is finite-sample estimation error.

\begin{table}[!htbp]
\centering
\caption{CUSP allocation of path energy in pure-mechanism populations, in percent. Each row changes only the stated population quantity. Entries are means \(\pm\) \(95\%\) confidence-interval half-widths over \(20\) repetitions. The theoretically correct term is bold.}
\label{tab:app_synthetic_mechanisms}
\small
\setlength{\tabcolsep}{4pt}
\begin{tabular}{lcc}
\toprule
Population change & Mean term & Deviation term \\
\midrule
Centers move & \textbf{99.89 \(\pm\) 0.02} & 0.11 \(\pm\) 0.02 \\
Covariances change & 1.05 \(\pm\) 0.40 & \textbf{98.95 \(\pm\) 0.40} \\
Prevalences change & \textbf{99.95 \(\pm\) 0.01} & 0.05 \(\pm\) 0.01 \\
\bottomrule
\end{tabular}
\end{table}

\subsection{DWUG implementation and decomposition}
\label{app:dwug_details}

The English and German DWUG evaluations \citep{schlechtweg2021dwug} use $9{,}107$ and $9{,}125$ XL-LEXEME usage embeddings \citep{cassotti2023xllexeme}, respectively. English contains $46$ targets with $61$--$100$ usages per target--period. German contains $50$ targets with $39$--$100$. For every target and period, BIC selects $K\in\{1,\ldots,5\}$ and diagonal or full covariance, jointly with covariance regularization in $\{10^{-6},10^{-5},10^{-4},10^{-3}\}$. Components with fitted weight below $0.01$ are excluded, and all fits use seed $0$. The DWUG vectors retain \(1{,}024\) coordinates after removal of the leading direction; no additional low-dimensional PCA projection is applied. Since the embedding dimension exceeds the number of usages per target--period, unregularized empirical full covariances are rank deficient. Positive covariance regularization makes the fitted Gaussian covariances nonsingular, but does not remove the statistical uncertainty of estimating them in this regime. The Gaussian recovery theorem assumes a correctly specified, fixed-dimensional model and the stated estimation conditions. In the lexical experiments, GMMs are regularized approximations: the decomposition remains exact for the fitted process, but the theorem does not guarantee recovery for normalized lexical data.

%%%% BEFORE
%Components with fitted weight below $0.01$ are excluded, and all fits use seed $0$.

%%%% BEFORE
%The primary $k=1$ representation removes the leading pooled panel direction before $\ell_2$ normalization.

For each language, we pool all target--period embeddings, fit one centering and principal-direction transform, and hold it fixed across targets. The primary $k=1$ representation removes the leading pooled panel direction before $\ell_2$ normalization. Across the lexical experiments, PCA fitting uses random seed $0$, and preprocessing fitting and transformation use single-threaded BLAS execution to control numerical variability. All methods in Table~\ref{tab:dwug_results} use these same preprocessed embeddings. CUSP selects its mixture configuration by unsupervised BIC; baseline parameters follow the fixed configurations supplied with the code. These are different parameter-selection rules, not matched hyperparameter-search budgets. The $k=0$ and $k=3$ representations are sensitivity checks. Confidence intervals in Table~\ref{tab:dwug_results} use $3{,}000$ paired bootstrap samples of target words, and Table~\ref{tab:app_dwug_decomposition} reports all three preprocessing settings.

\begin{table*}[!htbp]
\centering
\caption{Spearman correlation with DWUG graded-change judgments for the full, mean, and deviation traces. Each transform is fitted once to the corresponding language panel.}
\label{tab:app_dwug_decomposition}
\small
\setlength{\tabcolsep}{4pt}
\begin{tabular}{@{}lccccccccc@{}}
\toprule
& \multicolumn{3}{c}{$k=0$} & \multicolumn{3}{c}{$k=1$} & \multicolumn{3}{c}{$k=3$} \\
\cmidrule(lr){2-4}\cmidrule(lr){5-7}\cmidrule(l){8-10}
Language & Full & Mean & Dev. & Full & Mean & Dev. & Full & Mean & Dev. \\
\midrule
English & $.772$ & $.762$ & $.673$ & $.746$ & $.742$ & $.681$ & $.762$ & $.757$ & $.706$ \\
German & $.857$ & $.857$ & $.722$ & $.815$ & $.809$ & $.674$ & $.804$ & $.798$ & $.660$ \\
\bottomrule
\end{tabular}
\end{table*}

The mean trace remains close to the full trace under every preprocessing choice. The deviation trace is weaker but remains associated with the human score, and is comparatively large for heterogeneous words such as English \emph{bar} and German \emph{Rezeption}. The decomposition therefore identifies within component reorganization that the scalar full trace would otherwise conceal.

\subsection{Janus construction and supplementary results}
\label{app:janus_details}

\citet{cassotti2025sense} generated the released Janus records with Llama~3 conditioned on a lemma, an Oxford English Dictionary sense definition, and a year. The release contains $131{,}811$ sentences, $13{,}762$ sense--year records, and $2{,}768$ lemmas. We select $48$ lemmas with four released sense pools and ten sentences per pool, yielding $1{,}920$ fixed sentences. These populate six periods of $40$ usages per lemma by sampling with replacement. The resulting $11{,}520$ usage slots are not independent new sentences. Repeated contexts limit source diversity and can affect covariance estimates despite regularization. Sampling with replacement does not itself violate independence, but the enlarged set of usage slots should not be interpreted as an equally large collection of distinct historical observations. Resampling preserves any biases or irregularities in the released pools. The experiment tests recovery of planted distributional changes and compositional coherence, not historical fidelity or recovery of the released sense identities.

%%%% BEFORE
%The resulting $11{,}520$ usage slots are not independent new sentences.

Stable schedules keep the four pool prevalences at $(.25,.25,.25,.25)$. In gradual schedules, the old and new pool prevalences follow $(.90,.72,.54,.36,.18,0)$ and its reverse. Abrupt schedules use $(.90,.90,.90,0,0,0)$ and its reverse. The two background pools have nominal prevalence $.05$ in changed schedules, with finite counts allocated by largest remainder. Sixteen lemmas are assigned to each schedule. Released sense identifiers determine only these planted prevalences and are discarded before fitting. CUSP receives neither sense identifiers nor cross-period usage links.

We fit one pooled panel $k=1$ transform to the original $1{,}920$ contexts and freeze it before constructing schedules. A separate diagonal GMM is then fitted to every lemma--period panel, with BIC choosing $K\in\{1,\ldots,5\}$ and covariance regularization in $\{10^{-5},10^{-4}\}$ subject to minimum component weight $.01$ and at least three samples per component.

To test temporal coherence, pairwise OT independently optimizes every period-pair plan \citep{montariol2021transport,kishino2025quantifying}. For each consecutive triple, we compare its direct plan $q_{t,t+2}$ with the Markov composition of $q_{t,t+1}$ and $q_{t+1,t+2}$ and average the resulting $\ell_1$ disagreement. CUSP defines the non-adjacent plan by this composition, so its error is numerical zero by construction. Table~\ref{tab:app_janus_results} separates schedule recovery from the composition check.

\begin{table}[!htbp]
\centering
\caption{Janus recovery with unsupervised period local GMMs and pooled panel $k=1$ preprocessing. Composition error is mean $\ell_1$ disagreement.}
\label{tab:app_janus_results}
\small
\setlength{\tabcolsep}{4pt}
\begin{tabular}{@{}lclcc@{}}
\toprule
Schedule & $n$ & Schedule recovery & Pairwise OT & CUSP \\
\midrule
Stable & $16$ & Path length $0$ & $9.8\!\times\!10^{-18}$ & $1.0\!\times\!10^{-17}$ \\
Gradual & $16$ & Profile TV $.044$ & $.253$ & $1.8\!\times\!10^{-17}$ \\
Abrupt & $16$ & Peak accuracy $1.000$ & $6.4\!\times\!10^{-18}$ & $1.1\!\times\!10^{-17}$ \\
Overall & $48$ & $\rho=.954$, AUROC $1.000$ & $.084$ & $1.3\!\times\!10^{-17}$ \\
\bottomrule
\end{tabular}
\end{table}

The pairwise inconsistency is concentrated in gradual replacement, where mass is redistributed over several adjacent steps. Stable schedules and the single-boundary abrupt schedules already compose to numerical precision. Janus therefore does not establish uniform superiority over pairwise OT. It shows accurate recovery of the planted temporal profiles and verifies the coherence added by CUSP when change unfolds over multiple periods.

\subsection{Court corpus construction and fitting}
\label{app:court_data}

The Court panel is built from the May 6, 2024 snapshot of CourtListener's bulk opinion data, maintained by Free Law Project \citep{free_law_project2024courtlistener}, and includes opinions filed from 1950 through 2024. Excerpt identifiers are the \texttt{id} values of individual CourtListener opinion records in that snapshot. The target vocabulary contains $60$ unigram legal lemmas selected from the tables of contents of West Academic law textbooks.\footnote{\url{https://www.westacademic.com/}} The vocabulary was fixed before model fitting and ranking, and the complete list is included with the released code. The analysis uses eight decade bins from $1950$ through $2020$. The final bin contains opinions from \(2020\) through the May 2024 snapshot, while the preceding bins cover ten years each. Equalizing usage counts does not equalize calendar coverage. The reported distances compare pooled period distributions rather than annualized rates of semantic change, so the shorter final bin should be interpreted with this limitation.

%%% BEFORE
%The final bin contains opinions from \(2020\) through \(2024\), while the preceding bins cover ten years each.

%%%% BEFORE
%The same prompt and decision rule are applied to every lemma and decade.
Construction proceeds in five stages. First, we locate target-lemma matches and retain a window of approximately $200$ words on each side. Up to $5{,}000$ candidate opinions are considered for each lemma--decade. At most $20$ occurrences are scanned for an opinion--lemma pair, and at most six passages from that pair can be retained. Second, a fixed LLM prompt receives the marked occurrence and its surrounding window. It retains the occurrence only when the court is defining, applying, limiting, or otherwise construing the lemma as an object of legal reasoning. Captions, repeated procedural text, proper nouns, ordinary uses, and incidental mentions are discarded. The same prompt and decision rule are applied to every lemma and decade. Consequently, the analysis concerns the selected legal-reasoning usages, not changes in a word's prevalence between ordinary discourse and legal reasoning or its unrestricted general-language history. The exact prompts are included with the released code. The collection run used \texttt{gpt-5.6-luna} for both this retention decision and the subsequent verbatim extraction. It targeted up to $240$ distinct retained opinions per lemma--decade, with a minimum target of $120$. Third, a separate prompt copies one contiguous passage verbatim that contains the marked occurrence and completes the relevant reasoning. It does not summarize or rewrite the opinion. Fourth, XL-LEXEME embeds the marked target token in that extracted passage. Finally, if $N_{w,t}$ is the retained count for lemma $w$ in decade $t$, we sample $n_w=\min_tN_{w,t}$ passages from every decade with seed $0$. Each lemma therefore has equal counts across its eight periods, while $n_w$ may differ between lemmas.

The procedure yields $322{,}222$ passages before equalization, and the balanced panel contains $258{,}480$. The full balanced panel is used once to fit the Court centering and leading principal direction. The same fitted $k=1$ transform is then applied to every lemma and decade. For each lemma--decade, BIC selects a diagonal Gaussian mixture with $K\in\{1,\ldots,5\}$ and minimum fitted weight $.02$. Of the $480$ fits, $471$ select $K=5$, six select $K=4$, two select $K=3$, and one selects $K=2$. Because BIC often selects the upper bound, the component cap determines the resolution of the fitted representation. The resulting components describe recurring usage structure and are not claimed to recover a unique inventory of linguistic senses. Alternative caps may change component assignments and the allocation between mean and deviation terms, while the decomposition remains exact for each fitted process.

Across panel fitted ABTT settings \(k=0,\ldots,4\), a sensitivity analysis using a common \(z\geq1.1\) screen places \emph{privacy} among the top five lemmas by enrichment and the top two by secondary-mode temporal separation. Its largest change peaks in either \(1950{\rightarrow}1960\) or \(1960{\rightarrow}1970\). Thus its prominence and multi-period character persist across these preprocessing choices, although the decade of its largest peak varies.

%This concentration at the search boundary makes the choice of cap a substantive modeling limitation. Components should be interpreted at the resolution of the bounded Gaussian approximation, rather than as an estimate of the true number of linguistic senses. Changing the cap may change component assignments and the mean--deviation allocation, even though the decomposition remains exact for each fitted process.

%%%% BEFORE
%Of the $480$ fits, $471$ select $K=5$, six select $K=4$, two select $K=3$, and one selects $K=2$.

The released materials include the exact retention and extraction prompts, fixed vocabulary, configuration files, frozen preprocessing and model outputs, and scripts for reproducing the reported geometry, rankings, passage retrieval, and figures. Because the complete Court panel is large, the accompanying code package provides the frozen transform fitted on the full panel, processed inputs for the two main-text lemmas, and numerical outputs for all 60 lemmas. The released scripts reproduce the reported Court figures and rankings and rerun the main-text analyses. This subsection documents the corpus construction and fitting procedure for rebuilding the full panel from the cited CourtListener source. The code and released results are available at \url{https://github.com/hisanor013/cusp}. The complete balanced Court text and embedding panels will be made publicly available in a separate data release. Table~\ref{tab:court_data_counts} summarizes the resulting corpus and balanced panel.

%%%% BEFORE
%The raw CourtListener snapshot, intermediate retention outputs, and data-collection orchestration are not included in the supplementary materials. Reconstructing the balanced panel from raw opinions therefore additionally requires the same source snapshot and LLM decoding stack.

\begin{table}[!htbp]
\centering
\caption{Court corpus and balanced analysis-panel sizes. Per-decade counts are calculated after equalization within each lemma.}
\label{tab:court_data_counts}
\small
\setlength{\tabcolsep}{5pt}
\begin{tabular}{lr}
\toprule
Quantity & Count \\
\midrule
Distinct court opinions represented & $100{,}473$ \\
Target lemmas & $60$ \\
Retained passages before equalization & $322{,}222$ \\
Passages in the balanced analysis panel & $258{,}480$ \\
Median passages per lemma--decade & $531$ \\
%%%% BEFORE
%Range per lemma--decade & $294$--$911$ \\
Range of equalized counts across lemmas & $294$--$911$ \\
Passages per decade for \emph{privacy} & $651$ \\
Lemmas meeting the standardized-change threshold & $53$ \\
\bottomrule
\end{tabular}
\end{table}

\subsection{Court rankings and passage analysis}
\label{app:court_audits}

This subsection follows the Court analysis in the order in which every result is selected. The within-lemma standardized profile first identifies an adjacent decade that is unusually large relative to the lemma's own history. At that peak, $E_{\max}$ identifies the component movement whose share of change is largest relative to its transported mass. The word-local analysis then asks whether the leading modes peak at different transitions. Representative source and target passages make these numerically selected component movements inspectable.

All numerical selections precede passage inspection. After the analysis fixes a component movement, each usage is assigned to the fitted component with the highest posterior probability. Representative passages within the selected source and target components are then ranked independently using a fixed score based on proximity to the component center and a surface form check favoring exact matches and forms beginning with the target lemma. The retrieved passages provide representative source and target usages for the attributed movement, and CourtListener identifiers link every excerpt to its underlying opinion. Named decisions situate the retrieved language historically.

\noindent\textbf{Peak transition attribution by maximum enrichment.} At a lemma's selected peak, combine the mean and deviation contributions as
\[
a_{k\ell}=q_{k\ell}\operatorname{tr}\!\left(\mathbf C_{\mathrm{mean}}(k,\ell)+\mathbf C_{\mathrm{dev}}(k,\ell)\right),
\qquad \alpha_{k\ell}=a_{k\ell}/d_t^2.
\]
We define
\[
E_{\max}^{(w)}=\max_{q_{k\ell}>0}\frac{\alpha_{k\ell}}{q_{k\ell}},
\]
which identifies the component movement whose contribution share is most disproportionate to its transported-mass share. Since $a_{k\ell}=q_{k\ell}\operatorname{tr}(\mathbf C_{\mathrm{mean}}+\mathbf C_{\mathrm{dev}})$, the factor $q_{k\ell}$ cancels in the ratio. Thus $E_{\max}$ measures displacement contribution per unit of transported mass, normalized by the total squared change; there is no uncancelled inverse-mass factor. Numerically, the implementation excludes pairs with $q_{k\ell}\leq10^{-15}$ and floors the total squared change at $10^{-15}$. This is a per-unit-mass attribution statistic, not a significance measure, and rare component pairs can still have uncertain estimated displacements.  Because mixtures are fitted separately by period, $k$ and $\ell$ are period local labels. Equal or unequal indices have no intrinsic semantic interpretation. Table~\ref{tab:app_court_enrichment} reports the ten highest values. The complete ranking of all $53$ lemmas passing this threshold and the corresponding figure are included in the supplementary material.

%%%% BEFORE
%Thus $E_{\max}$ measures displacement contribution per unit of transported mass, normalized by the total squared change.

\begin{table*}[!htbp]
\centering
\caption{Ten highest maximum-enrichment transitions under panel-fitted $k=1$. Contribution $\alpha$ is the selected pair's share of variation at the selected peak.}
\label{tab:app_court_enrichment}
\small
\setlength{\tabcolsep}{4pt}
\begin{tabular}{@{}clcccc@{}}
\toprule
Rank & Lemma & Peak transition ($z$) & Pair & $\alpha$ & $E_{\max}$ \\
\midrule
1 & privacy & $1960{\to}1970$ ($1.43$) & $2{\to}3$ & $.371$ & $15.10$ \\
2 & search & $1980{\to}1990$ ($1.55$) & $3{\to}4$ & $.690$ & $13.67$ \\
3 & standing & $1960{\to}1970$ ($1.63$) & $3{\to}4$ & $.366$ & $13.23$ \\
4 & warrant & $1950{\to}1960$ ($2.26$) & $1{\to}0$ & $.335$ & $11.11$ \\
5 & trespass & $2000{\to}2010$ ($1.74$) & $1{\to}4$ & $.390$ & $10.51$ \\
6 & exercise & $1960{\to}1970$ ($1.36$) & $0{\to}1$ & $.587$ & $10.33$ \\
7 & effect & $1970{\to}1980$ ($1.53$) & $0{\to}2$ & $.239$ & $10.17$ \\
8 & due & $1960{\to}1970$ ($1.22$) & $3{\to}0$ & $.158$ & $9.25$ \\
9 & official & $1950{\to}1960$ ($2.06$) & $2{\to}3$ & $.119$ & $7.60$ \\
10 & liberty & $1970{\to}1980$ ($1.50$) & $2{\to}0$ & $.169$ & $7.16$ \\
\bottomrule
\end{tabular}
\end{table*}

The following paragraphs examine the five highest ranked lemmas in Table~\ref{tab:app_court_enrichment}. This selection is determined by the numerical ranking rather than by the readability of the retrieved passages. The complete ranking and component assignments are included in the supplementary material.

\textbf{Privacy.} The source says that ``Professor Prosser classified these types of cases into four categories'' (opinion~\texttt{1179312}), while the target says that warrantless inspection poses ``only a minimal threat to justifiable expectations of privacy'' (opinion~\texttt{1602980}). The pair moves from Prosser-style tort classification \citep{prosser1960privacy} toward post-\emph{Katz} expectation-of-privacy language \citep{katz1967unitedstates}. The preceding decade's leading pair moves from privacy as ``a direct wrong of a personal character'' (opinion~\texttt{2604478}) toward the same Prosser component, making tort articulation, systematic classification, and expectation-of-privacy language successive rather than one retrospectively chosen contrast.

\textbf{Search.} The source applies a workers' compensation rule requiring a ``conscientious work search'' and asks whether an unsuccessful search reflects inability to work (opinion~\texttt{1139619}). Targets describe a ``pat-down search'' and a ``search of a jacket'' (opinion~\texttt{197788}). The transition is from \emph{search} as a legally required effort to obtain work to physical inspection. The source is retained because work search is itself part of the legal disability test, not incidental job-seeking narration.

\textbf{Standing.} A source states that ``Standing trees are a part of the realty'' (opinion~\texttt{1331324}). Targets ask whether a plaintiff has the ``requisite standing to maintain the action'' and describe standing as a federal-court question under Article~III (opinions~\texttt{1513512}, \texttt{357908}). The pair separates the adjectival property sense from the jurisdictional doctrine.

\textbf{Warrant.} Sources state that facts ``may warrant an inference of negligence'' or ``warrant an allowance of punitive damages'' (opinions~\texttt{2083592}, \texttt{1437656}). Targets discuss ``probable cause for the issuance of a search warrant'' and a ``search warrant'' (opinions~\texttt{2216781}, \texttt{2167941}). Both pairs move from the verb licensing a conclusion to the noun denoting a judicial instrument.

\textbf{Trespass.} Sources plead a common-law trespass claim or state that ``a cognizable claim for trespass occurs when personal property \ldots{} is injured or taken'' (opinions~\texttt{2421595}, \texttt{1557045}). Targets describe a ``trespass-to-try-title action'' or an ``action in trespass to try title'' (opinions~\texttt{4293854}, \texttt{1889942}). The retrieved language supports redistribution between distinct legal constructions of \emph{trespass}, without assigning it to a landmark decision.

\noindent\textbf{Temporal separation of word-local modes.} Let $e_r$ be mode $r$'s share of word-local mean-path energy. Among the five modes with the largest path-energy shares, order the two leading shares as $e_{(1)}\geq e_{(2)}$ and denote their peak decades by $\tau_{(1)}$ and $\tau_{(2)}$. For lemmas passing the standardized-change threshold, define
\[
S_2=e_{(2)}\log(1+E_{\max})
\]
when $\tau_{(1)}\neq\tau_{(2)}$. If the two leading modes peak together, the profile is classified as synchronized and receives no $S_2$ rank. Multiplying the second mode's energy share by a monotone transform of $E_{\max}$ favors histories with both a substantial second direction and a clearly attributable movement. The score is fixed only to prioritize passage analyses of temporally separated histories. It does not estimate the model, test significance, or determine whether change occurred. Table~\ref{tab:app_court_secondary_mode} reports the ten highest values.

\begin{table}[!htbp]
\centering
\caption{Ten highest secondary-mode temporal-separation scores. $e_{(1)}$ and $e_{(2)}$ are the two leading shares among the five highest-energy mean-path modes.}
\label{tab:app_court_secondary_mode}
\small
\setlength{\tabcolsep}{4pt}
\begin{tabular}{@{}clcccc@{}}
\toprule
Rank & Lemma & $S_2$ & $e_{(1)}$ & $e_{(2)}$ & Mode peaks \\
\midrule
1 & privacy & $.837$ & $.391$ & $.301$ & $1960{\to}1970$ / $1950{\to}1960$ \\
2 & standing & $.788$ & $.406$ & $.297$ & $1960{\to}1970$ / $1970{\to}1980$ \\
3 & effect & $.721$ & $.452$ & $.299$ & $1990{\to}2000$ / $1970{\to}1980$ \\
4 & trespass & $.707$ & $.359$ & $.289$ & $1970{\to}1980$ / $1990{\to}2000$ \\
5 & taking & $.690$ & $.387$ & $.360$ & $2010{\to}2020$ / $2000{\to}2010$ \\
6 & due & $.683$ & $.371$ & $.294$ & $1980{\to}1990$ / $1990{\to}2000$ \\
7 & search & $.655$ & $.481$ & $.244$ & $1980{\to}1990$ / $1990{\to}2000$ \\
8 & cause & $.610$ & $.387$ & $.318$ & $1950{\to}1960$ / $1970{\to}1980$ \\
9 & content & $.581$ & $.372$ & $.308$ & $1960{\to}1970$ / $2000{\to}2010$ \\
10 & burden & $.567$ & $.328$ & $.308$ & $1970{\to}1980$ / $1990{\to}2000$ \\
\bottomrule
\end{tabular}
\end{table}

Privacy, standing, trespass, and search also appear in the enrichment ranking, where their selected transitions are grounded by the passage analyses above. The high positions of \emph{effect} and \emph{taking} are consistent with broad, polysemous histories. We retain their quantitative positions without assigning them a single landmark-case narrative. Complete values for all $41$ separated profiles are included in the supplementary material.

\textbf{Privacy modes.} Mode~2 ($30\%$) peaks at $1950{\rightarrow}1960$. Its leading pair moves from a right of privacy ``to be subordinated'' to the public interest (opinion~\texttt{1424064}) toward an ``actionable invasion'' of privacy, and the action ``sounds in tort'' (opinion~\texttt{1741636}). Another target locates eavesdropping privacy in the Fourth Amendment (opinion~\texttt{255989}). Modes~3 and~4 ($15\%$ and $11\%$) peak at $1960{\rightarrow}1970$ and are led by the Prosser-to-expectations transition. Mode~4 also connects \textit{Warden v. Hayden}'s statement \citep{warden1967hayden} that the Fourth Amendment protects the right of privacy, ``rather than any interest in the property seized'' (opinion~\texttt{1982603}), to ``no legitimate expectation of privacy'' in a cargo hold (opinion~\texttt{1378708}). Mode~1 ($39\%$) peaks at $1960{\rightarrow}1970$ and remains strong at $1970{\rightarrow}1980$. At the later step, one pair moves from the ``rights of privacy of practitioners of the healing arts and of their patients'' (opinion~\texttt{2597129}) toward statutory balancing of ``personal privacy'' against the ``public interest'' (opinion~\texttt{1677943}). Another connects family privacy under \emph{Griswold} and \emph{Roe} (opinion~\texttt{2142114}) to ``no legitimate expectation of privacy'' in plain view (opinion~\texttt{2070690}) \citep{griswold1965connecticut,roe1973wade}. The modes therefore recover staggered and partly concurrent tort, constitutional, informational, and search-privacy usages rather than one undifferentiated event.

\textbf{Expectation modes.} The three leading modes carry $48\%$, $26\%$, and $18\%$ of mean-path energy and all peak at $1960{\rightarrow}1970$. Their deterministically selected passages move from ``an expectation of pay or compensation'' and contributions made ``without expectation of services or direct benefits'' (opinions~\texttt{2226138}, \texttt{1636296}) toward ``a reasonable expectation of privacy infringed by the search and seizure'' and an area ``protected by an expectation of privacy'' (opinions~\texttt{1365600}, \texttt{1311660}). The contrast with \emph{privacy} therefore comes from a fixed temporal criterion rather than retrospective case selection.

\end{document}